\documentclass[10pt,twocolumn,letterpaper]{article}

\usepackage[pagenumbers]{cvpr} 

\usepackage{graphicx}
\usepackage{amsmath}
\usepackage{amssymb}
\usepackage{booktabs}

\usepackage{amsmath}
\usepackage{amssymb,comment}
\usepackage{mathtools}
\usepackage{amsthm,color,pstricks}
\usepackage[ruled,vlined]{algorithm2e}

\usepackage{subcaption}

\newtheorem{theorem}{Theorem}
\newtheorem{lemma}{Lemma}

\newtheorem{definition}{Definition}

\newtheorem{remark}{Remark}
\newtheorem{assumption}{Assumption}

\def\PP{\mathbb{P}}
\def\RR{\mathbb{R}}

\def\cA{\mathcal{A}}
\def\cD{\mathcal{D}}

\def\cL{\mathcal{L}}

\def\cP{\mathcal{P}}
\def\cS{\mathcal{S}}

\def\cU{\mathcal{U}}
\def\cW{\mathcal{W}}

\newcommand\myas{\stackrel{\mathclap{\normalfont\mbox{a.s.}}}{=}}

\DeclareCaptionLabelFormat{andtable}{#1\ #2}

\usepackage[pagebackref,breaklinks,colorlinks]{hyperref}

\usepackage[capitalize]{cleveref}
\crefname{section}{Sec.}{Secs.}
\Crefname{section}{Section}{Sections}
\Crefname{table}{Table}{Tables}
\crefname{table}{Tab.}{Tabs.}

\def\cvprPaperID{11713} 
\def\confName{CVPR}
\def\confYear{2022}

\begin{document}

\title{Deep Unlearning via Randomized Conditionally Independent Hessians}

\author{Ronak Mehta\thanks{Joint First Authors.} \textsuperscript{\rm 1}\\
{\tt\small ronakrm@cs.wisc.edu}
\and
Sourav Pal\textsuperscript{*\rm 1}\\
{\tt\small spal9@wisc.edu}
\and
Vikas Singh\textsuperscript{\rm 1} \\
{\tt\small vsingh@biostat.wisc.edu}
\and
Sathya N. Ravi\textsuperscript{\rm 2} \\
{\tt\small sathya@uic.edu}
\and
{\textsuperscript{\rm 1}University of Wisconsin-Madison $\quad$
\textsuperscript{\rm 2}University of Illinois at Chicago}
}

\maketitle

\begin{abstract}
Recent legislation has
led to interest in {\em machine unlearning}, i.e., removing specific training samples from a {\em predictive} model as if they never existed in the training dataset. 
Unlearning may also be required due to  corrupted/adversarial data or simply a user's updated privacy requirement.
For models which require no training ($k$-NN), 
simply deleting the closest original sample can be effective. 
But this idea is inapplicable to models which learn richer 
representations.
Recent ideas leveraging optimization-based updates
scale poorly with the model dimension $d$,  
due to 
inverting the Hessian of the loss function. 
We use a variant of a new conditional independence coefficient, 
L-CODEC, to identify a subset of the model parameters with the most semantic overlap on an individual sample level. 
Our approach completely avoids the need to invert a (possibly) huge matrix. 
By utilizing a Markov blanket selection, we premise that L-CODEC is also suitable for deep unlearning,
as well as other applications in vision.
Compared to alternatives, L-CODEC makes approximate unlearning possible 
in settings that would otherwise be infeasible, including vision models used for face recognition, person re-identification and NLP models that may require unlearning samples identified for exclusion.
Code is available at \url{https://github.com/vsingh-group/LCODEC-deep-unlearning}
\end{abstract}

\section{Introduction}
As personal data becomes a valuable commodity, legislative efforts have begun to push back on its widespread collection/use particularly for training ML models. Recently, a focus is the ``right to be forgotten" (RTBF), i.e., the right of an individual's data to be deleted from a database (and derived products).
Despite existing legal frameworks on fair use, industry scraping has led to personal images being used without consent, e.g. \cite{Exposing}.
Large datasets are not only stored for descriptive statistics, but used in training large models.
While regulation (GDPR, CCPA) has not specified the extent to which data must be forgotten, it poses a clear question: is  deletion of the data enough, or does a model trained on that data also needs to be updated?

Recent work by \cite{carlini2019secret,carlini2020attack} has identified scenarios where trained models are vulnerable to attacks that can reconstruct input training data. More directly, recent rulings by the Federal Trade Commission \cite{ftc,ftc2} have ordered companies to fully delete and destroy not only data, but also any model trained using those data.
While deletion and (subsequent) full model retraining without the deleted samples is possible, most in-production models require weeks of 
training and review, with extensive computational/human resource cost. With additional deletions, it is infeasible to retrain each time a new delete request comes in. 
So, how to update a model ensuring the data is deleted without retraining?

{\bf Task.} Given a set of input data $\cS: \{z_i\}_{i=1}^n \sim \mathcal{D}$ of size $n$, training simply identifies a hypothesis $\hat{w} \in \cW$  via an iterative scheme $w_{t+1} = w_t - g(\hat{w},z')$ until convergence, where $g(\cdot,z')$ is  a stochastic gradient of a fixed loss function. Once a model at convergence is found, \textit{machine unlearning} aims to identify an update to $\hat{w}$ through an analogous {\em one-shot unlearning update}:
\begin{align}\label{eq:unlearn}
    w' = \hat{w} + g_{\hat{w}}\left(z'\right),
\end{align}
for a {\em given} sample $z' \in \cS$ that is to be {\bf unlearned}.

\noindent\textbf{Contributions.} We address several computational issues with existing approximate formulations for unlearning by taking advantage of a new statistical scheme for sufficient parameter selection. 
First, in order to ensure that a sample's impact on the model predictions is minimized, we propose a measure for computing conditional independence called L-CODEC which  identifies the Markov Blanket of parameters to be  updated. 
Second, we show that the L-CODEC identified Markov Blanket enables unlearning in previously infeasible deep models, scaling to networks with hundreds of millions of parameters. 
Finally, we demonstrate the ability of L-CODEC to unlearn samples and entire classes on networks, from CNNs/ResNets to transformers, including face recognition and person re-identification models.

\begin{figure*}
    \centering
    \includegraphics[height=1.5in,trim={0 12cm 5cm 2.5cm},clip]{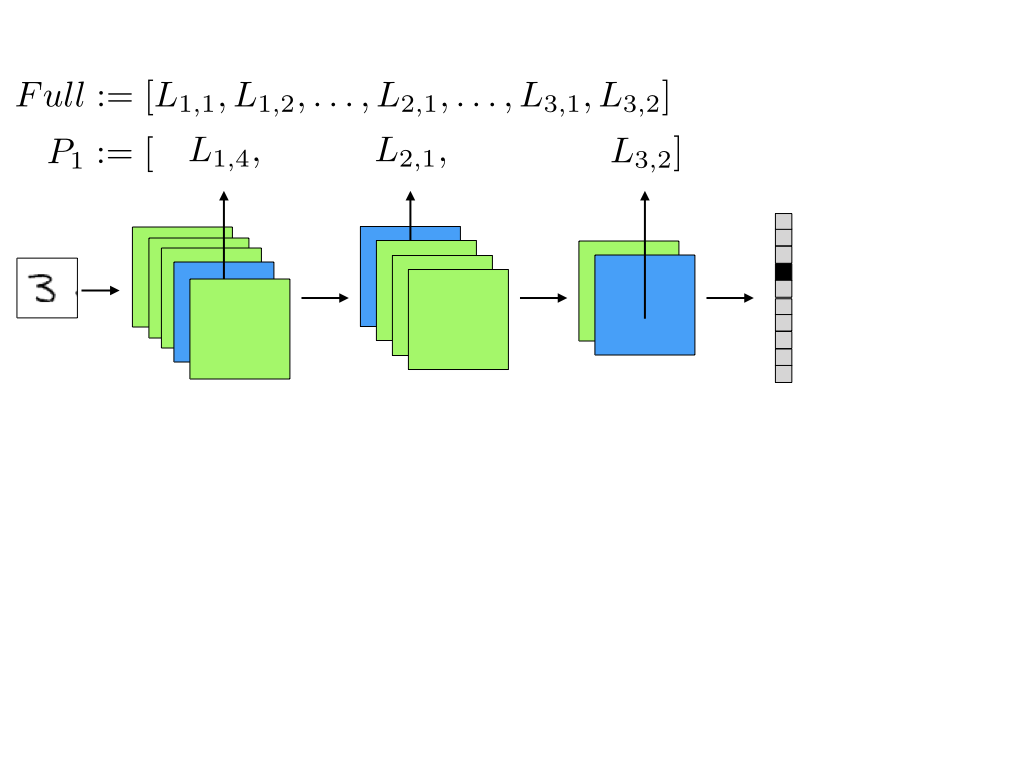}
    \includegraphics[height=1.5in,trim={0 1cm 0 0},clip]{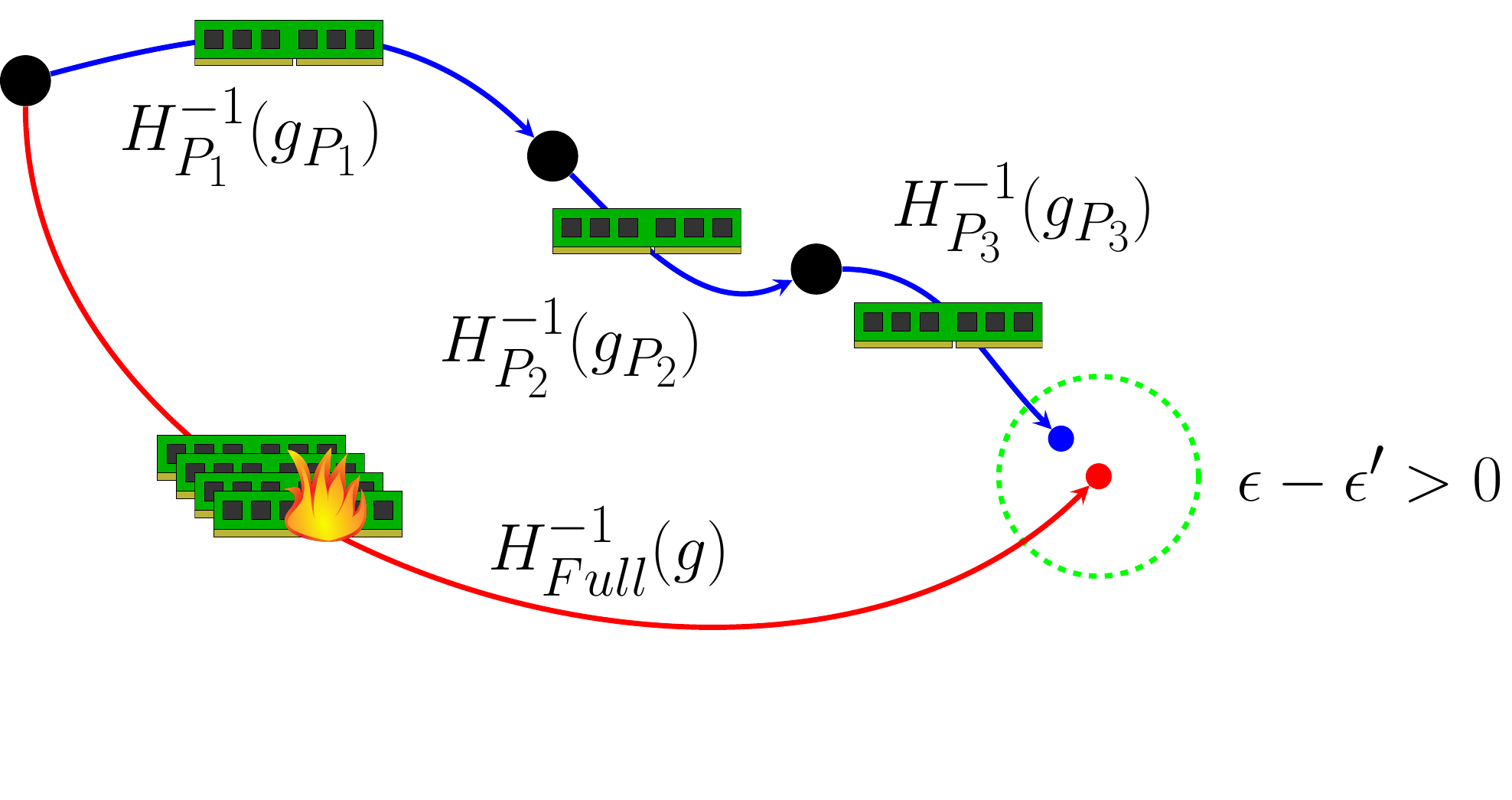}
    \vspace{-20pt}
    \caption{\label{fig:main}Large deep learning networks typically associate specific subsets of network parameters, blocks (blue), to specific samples in the input space.
    Traditional forward or backward passes may not reveal these blocks: high correlations among features may not distinguish important ones. Input perturbations can be used to identify them in a probabilistic, distribution-free manner. These blocks can then be unlearned together in an efficient block-coordinate style update (right, blue lines), approximating an update to the full network which requires a costly/infeasible full Hessian inverse (red line).}
\end{figure*}

\section{Problem Setup for Unlearning}
Let $\cA$ be an algorithm that takes as input a training set $\cS$ and outputs a hypothesis $w \in \cW$, defined by a set of $d$ parameters $\Theta$. 
An unlearning scheme $\cU$ takes as input a sample $z' \in \cS$ used as input to $\cA$, and ideally, outputs an {\bf updated} hypothesis $w' \in \cW$ where $z'$ has been deleted from the model.
%
%
An unlearning algorithm should output a hypothesis that is close or equivalent to one that would have been learned had the input to $\cA$ been $\cS \setminus z^\prime$. A framework for this goal was given by \cite{ginart2019making} as,
\begin{definition}[$(\epsilon,\delta)-$ forgetting]\label{def:forget}
For all sets $\cS$ of size $n$, with a ``delete request'' $z' \in \cS$, an unlearning algorithm $\cU$ is $(\epsilon, \delta)-$forgetting if
\begin{align}
\PP(\cU(\cA(S), z') \in \cW) \leq e^\epsilon\PP(\cA(\cS\setminus z') \in \cW) + \delta
\end{align}
\end{definition}
In essence, for an existing model $w$, a good unlearning algorithm for request $z' \in \cS$ will output a model $\hat{w}$ close to the output of $\cA(S \setminus z')$ with high probability.

\begin{remark}
Definition \ref{def:forget} is similar to the standard definitions of differential privacy. The connection to unlearning is: if an algorithm is $(\epsilon, \delta)-$forgetting for unlearning, then it is also differentially private. 
\end{remark}

If $\cA$ is an empirical risk minimizer for the loss $f$, let
\begin{align}
    \cA : (\cS, f) \rightarrow \hat{w}
\end{align}
$\hat{w} = \arg\min F(w)$ and $F(w) = \frac{1}{n}\sum_{i=1}^n f(w, z_i).$ 
Recall $g(z')$ from \eqref{eq:unlearn}:  
our unlearning task 
essentially involves 
identifying 
the form of $g(z')$ for which the update in \eqref{eq:unlearn} is $(\epsilon,\delta)$-forgetting. If an oracle provides this information, we have 
accomplished the unlearning task.

The difficulty, 
as expected, tends to 
depend on $f$ and $\cA$. 
Recent unlearning results have identified forms of $f$ and $\cA$ where such a $g(z')$ exists. The authors in \cite{sekhari2021remember} define $g(z') = \frac{1}{n-1}H'^{-1}\nabla f(\hat{w},z')$, where
\begin{align}\label{eq:sekhariunlearn}
    H' &= \frac{1}{n-1} \left(n\nabla^2 F(\hat{w}) - \nabla^2 f(\hat{w},z')\right),
\end{align}
with additive Gaussian noise $w' = w' + N(0,\sigma^2)$ scaling as a function of $n, \epsilon, \delta$, and the Lipschitz and (strong) convexity parameters of the loss $f$. We can interpret the update using \eqref{eq:sekhariunlearn} from the optimization perspective as a trajectory ``reversal": starting at a random initialization, the first order (stochastic gradient) trajectory of  ${w}$ (possibly) {\em with}  $z'$ is reversed using {\em residual} second order curvature information (Hessian) at the optimal $\hat{w}$ in \eqref{eq:sekhariunlearn}, achieving unlearning. This is shown to satisfy Def.~\ref{def:forget}, and only incurs an additive error that scales by $O(\sqrt{d}/n^2)$ in the gap between $F(w')$ and the global minimizer $F(w^*)$ over the ERM $F(\hat{w})$. 


{\bf Rationale for approximate schemes.} From the reversal of $w$ optimization perspective, it is clear that there may be other choices to achieve unlearning.
For a practitioner interested in unlearning, the aforementioned algorithm (as in \eqref{eq:sekhariunlearn}) can be directly instantiated if one has extensive computational resources.
Indeed, in settings where it is not directly possible to compute the Hessian inverse necessary for $H'^{-1} \nabla f(\hat{w},z')$, we must consider alternatives. 

{\bf A potential idea.} Our goal is to identify a form of $g(z')$ that {\bf approximates} $H'^{-1} \nabla f(\hat{w},z')$. 
Let us consider the Newton-style update suggested by 
\eqref{eq:sekhariunlearn}
as a smoothing of a traditional first order gradient step. 
The inverse Hessian is a weighting matrix, appropriately scaling the gradients based on the second order difference between the training set mean point $F(\hat{w})$ and at the sample of interest $f(\hat{w},z')$. 
This smoothing can also be seen from an information perspective: the Hessian in this case corresponds to a Fisher-style information matrix, and its inverse as a conditional covariance matrix \cite{Golatkar_2021_CVPR,golatkar2020forgetting}.
It is not hard to imagine that from this perspective, if there are {\em specific set of parameters} that have {\em small gradients} at $f(\hat{w},z')$ or if the information matrix is zero or small, then we need not consider their effect. 

{\bf Examples of this intuition in vision.} \cite{bau2017network,fong2018net2vec,Sun_2019_ICCV} and others have shown that models trained on complex tasks tend to \textit{delegate} subnetworks to specific regions of the input space. That is, parameters and functions within networks tend to (or can be encouraged to) act in \textit{blocks.}
For example, activation maps for different filters in a trained (converged) CNN model show differences for different classes, especially for filters closer to the output layer.
We formalize this observation as an assumption for samples in the training set.

\begin{assumption}\label{assum:sub}
For all subsets of training samples $S \subset \cS$, there exists a subset of trained model parameters $P^* \subset \Theta$ such that
\begin{align}\label{eq:assum}
    f(S) \bot w_{\Theta\setminus P}^* | w_{P}^*
\end{align}
\end{assumption}
Due to the computational issues discussed above, 
if we could make such a simple/principled selection scheme practical, it may offer significant 
benefits.


\section{Related Work}
To contextualize our contributions, 
we briefly review existing proposals for machine unlearning. 

\noindent\textbf{Na\"ive, Exact Unlearning.}
A number of authors have proposed methods for exact unlearning, in the case where $(\epsilon=0, \delta=0)$. SVMs by \cite{romero2007incremental,karasuyama2009multiple}, Na\"ive Bayes Classifiers by \cite{cao2015towards}, and $k$-means methods by \cite{ginart2019making} have all been studied. 
But these algorithms do not translate to stochastic models with millions of parameters.

\noindent\textbf{Approximate Unlearning.} 
With links to fields such as robustness and privacy, we see more developments in approximate unlearning under Definition~\ref{def:forget}. 
The so-called $\epsilon$-certified removal by \cite{guo2019certified} puts forth similar procedures when $\delta=0$, and the model has been trained in a specific manner.
\cite{guo2019certified,izzo2020approximate} provide updates to linear models and the last layers of networks, and 
\cite{golatkar2020forgetting,golatkar2020eternal} provide updates based on linearizations that work over the full network, and follow-up work by \cite{Golatkar_2021_CVPR} presents a scheme to unlearn under an assumption that some samples will not need to be removed.

Other recent work has taken alternative views of unlearning, which do not require/operate under probabilistic frameworks, see \cite{bourtoule2021machine,neel2021descent}. These schemes present good guarantees in the absolute privacy setting, but they require more changes to  pipelines (sharding/aggregating weaker models) and scale unsatisfactorily in large deep learning settings.

\section{Randomized Markovian  Block Coordinate Unlearning}
If there exist entries of the vector $g(z')= H'^{-1} \nabla f(\hat{w}, z')$ that we can, through {\em some} procedure, identify as zero, then we can simply avoid computing such zero coordinates. Not only can we zero out those particular entries in the inverse and the gradient, but we can take advantage of the blockwise inverse to {\em completely remove those parameters from all computations.} If possible, it would  immediately change the complexity from $O(d^3)$ to $O(p^3)$, where $p\ll d$ is the size of the subset of parameters that we know are \textit{sufficient} to update.

Let $P \subseteq \Theta:= \{1,\ldots, d\}$ be the index set of the parameters that are ``sufficient'' to update. A direct procedure may be to identify this subset $P$ with
\begin{align}\label{eq:normsel}
P = \underset{P \in \cP(\Theta)}{\arg\min} \ ||\tilde{w} - \tilde{w}_P||,
\end{align}
where $\cP(\Theta)$ is the {\em power set} of the elements in $\Theta$ and $\tilde{w}_P$ is the subset of the parameters we are interested in updating. 
Note that a simple solution to this problem {\em does} exist: choosing the $p=|P|$ parameters with the largest change will minimize this distance for typical norms. This can be achieved by thresholding the updates $g(z^\prime)$ for $\hat{w}$. However, this \textit{requires computing the full update for $g(z^\prime)$}. We want a preprocessing procedure that performs the selection \textit{before} computation of $g(z^\prime)$ is needed.

\noindent\textbf{A probabilistic angle for selection.} 
We interpret a deep network $\cW$ as a functional on the input space $\cD$. This perspective is common in statistics for variable selection (e.g., LASSO), albeit used {\em after} the entire optimization procedure is performed i.e., at the optimal solution. The only difference here is that we use it at approximately optimal solutions as given by ERM minimization. Importantly, this view allows us to identify regions in $\cW$ that contain the most information about a query sample $z'$. 
We will formalize this intuition using recent results for conditional independence (CI) testing.
Finding $w_P$ above should also satisfy
\begin{align}\label{eq:paramCI}
    z' \bot w_{\Theta\setminus P} | w_P
\end{align}
This CI formulation is well studied within graphical models. Many measures and hypothesis tests have been proposed to evaluate it.
The {\em coefficient of conditional dependence} (CODEC)  in \cite{codec}, along with their algorithm for ``feature ordering", FOCI, at first seems to offer a solution to  \eqref{eq:paramCI}, and in fact, can be implemented ``as is'' for shallow networks. (Review of other CI tests are in the appendix.)



\paragraph{Using CODEC directly for Deep Unlearning is inefficient.} There are two issues: First, when applying CODEC to problems with a very large $n$ with discrete values, the cost of tie-breaking for computing nearest neighbors can become prohibitive. Second, $z'$ is not a random variable for which we have a number of instances. We defer discussion of the second issue  to Section~\ref{sec:deepunlearn}, and address the first issue here.

Consider the case where a large number of elements have an equal value. With an efficient implementation using $kd-$trees, identifying the nearest-neighbor as required by CODEC would still require expanding the nodes of all elements with equal value. As an example, if we are looking for the nearest neighbor to a point at the origin and there are a large number of elements on the surface of a sphere centered at the origin, we still require checking all entries and expanding their nodes in the tree, even when we know that they are all equal for this purpose.

Interestingly, this problem has a relatively elegant solution.
We introduce a randomized version of CODEC, L-CODEC. For variables $A,B,C$:
\begin{align}\label{eq:lcodec}
T_L := T\left(\tilde{B}, \tilde{C} | \tilde{A} \right),
\end{align}
where $\tilde{B} = B + N(0,\sigma^2)$, and similarly for $\tilde{C}, \tilde{A}$.
This additive noise can simply be scaled to the inverse of the largest distance between any points in the set.
By requiring this noise to be smaller than any distance between items in the set, the ranking will remain the same between unique discrete values, and will be perturbed slightly for equal ones.
In expectation, this will still lead to the true dependence measure. The noise addition is consistent with the Randomization criterion for conditional independence -- for random variables $A, B, C$ in Borel spaces, $A\bot B | C$ \textit{iff} $A \myas h(B, U)$ for some measurable function $h$ and uniform random variable $U \sim \text{Uniform}(0,1)$ which is independent of $(B,C)$ as in \cite{ptheory}. 

\begin{remark}\label{rem:sens}
An altered version of this setup also gives us a form of explainability, where we can apply sensitivity analysis to each input feature or pixel and estimate its effect on the output via a similarly randomized version of the  Chatterjee rank coefficient $T(A,B)$, proposed by \cite{chatterjee2020new}.
\end{remark}



\subsection{Efficient Subset Selection that is also Sufficient for Predictive Purposes}
The above test is good for \eqref{eq:paramCI} if we know {\em which} subset $P\in \cP(\Theta)$ to test. 
Recent work by \cite{bullseye} proposes a selection procedure using an iterative scheme to slowly build the sufficient set, adding elements which maximally increase the information explained in the outcome of interest.
While it is efficient (polynomial in size), we must know the maximal degree. A priori, we may have no knowledge of what this size is, and for parameter subsets it may be very high.

\begin{figure}
    \centering
    \includegraphics[width=\columnwidth,trim={0 18cm 4cm 0},clip]{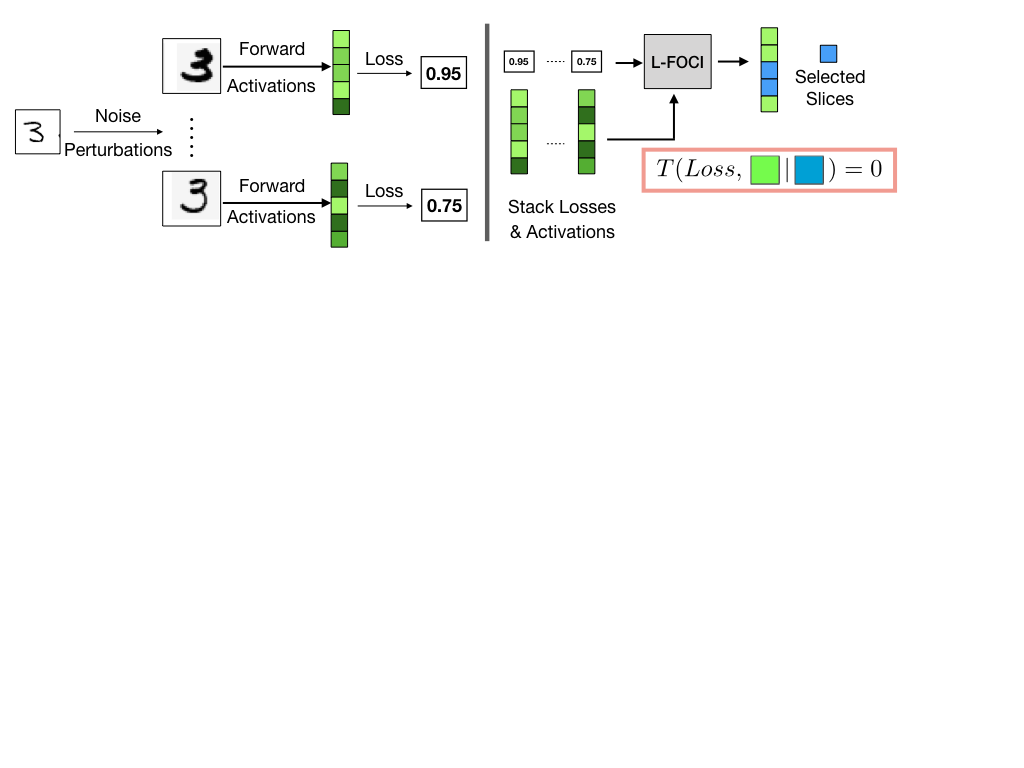}
    \caption{A sample is perturbed and passed through the network. Activations are aggregated alongside losses and fed to L-FOCI. Selected rows represent slices of the corresponding layer that are sufficient for unlearning.}
    \label{fig:lfoci}
\end{figure}

When using L-CODEC, we can use a more straightforward Markov Blanket identification procedure adapted from \cite{codec}. FOCI more directly selects which variables are valuable for explaining $z'$, and in fact, is proven to identify the sufficient set (Markov Blanket) with a reasonable number of samples. Briefly, in our L-FOCI, the sufficient set is built incrementally with successive calls to L-CODEC, moving the most ``dependent" feature from the independent set to the sufficient set.
See appendix for details.

\paragraph{Summary.} This procedure alleviates the first issue in terms of sufficient subset or Markov Blanket selection; compared to existing methods using information-theoretic measures that require permutation testing, L-FOCI directly estimates the change in variance when considering a proposal to add to the set.
Now, we discuss how this selection can help identify sets of parameters that can be updated.

\section{Deep Unlearning via L-FOCI Hessians}
\label{sec:deepunlearn}

Our input samples to scrub $z^\prime$ are not random variables for which we have samples or distributional assumptions, nor are our parameters. In this case, 
a perturbation-based scheme may be useful when attempting to generate samples for unknown distributions.

Considering Assump.~\ref{assum:sub}, when only some parameters are useful for the final outcome on an input sample $z^\prime \in S$, the effect of those parameters can be measured through activations due to the forward pass of a model. We estimate the conditional independence test in \eqref{eq:assum} through activations as
\begin{align}
    f(z^\prime) \bot a_{\Theta\setminus P}^* | a_P^*,
\end{align}
where $a_P$ for some parameter subset $P \subseteq \Theta$ is defined as the linear activations generated by the forward pass through the model.
This formulation relates to a generalized version of the solution in \S3 of \cite{bullseye}, where conditional mutual information is estimated via feature mappings.

As an example, if a network has linear layers $\cL$, a simple linear layer $l \in \cL$ with parameters $w_l \in \RR^{a \times b}$ would have activations $a_l \in \RR^b$, with $a_l = w_l a_{l-1}$. For each entry $a_{l,j}$ in the vector $a_l$, the associated parameters in the layer are $w_l[:,j]$. Thus, we break up the network into influential \textit{slices.} These slices can be seen as a finer view of the parameter space compared to typical layerwise selection, but coarser than a fully discrete one. Next, $\cL$ now refers to the collection of these slices, with a specific slice as $l$.

The tuple of variables we need samples from is now
\begin{align}
    \{a_1, \ldots a_{|\cL|}, \cL(z^\prime)\}
\end{align}
We can obtain samples from this set by perturbing the input and consecutively collecting activations along all weight slices during the computation of the loss. For a particular perturbation $\xi^j \sim N(0,\sigma^2)$,
\begin{align}
    x_i^j = x_i + \xi^j; \quad
    l^j, a_L^j = \{l(x_i^j), a_1^j, \ldots a_{|\cL|}^j\}
\end{align}
The tuples $(l^j, a_L^j)$ serve as samples for our conditional independence test, 
\begin{align}
    (P \subseteq \Theta) &= \mbox{L-FOCI}((l^j, a_\cL^j)_{j=1}^m)
\end{align}
for $J:=\{j \in 1,\ldots,m\}$ perturbations (see Figure~\ref{fig:lfoci}).

In Alg.~\ref{alg:blockunlearn}, the activations are collected using hooks within the forward pass. 
First, gradients at the last and penultimate epoch for full training are stored during the original training pass. Given a sample to unlearn, we compute L-FOCI over the perturbed activations and losses generated by the forward pass, and identify which parameter sets will be updated. We compute the approximate Hessian over these parameters via finite differences for both the full model and for the model only over the sample of interest. Finally, we apply the blockwise Newton update to the subset of parameters as in \eqref{eq:unlearn} with appropriate DP noise as in \cite{sekhari2021remember}.
\begin{algorithm}
\small
\SetAlgoLined
\KwData{A trained model $\hat{w}$, gradient vectors $\nabla_1 F(\hat{w}), \nabla_2 F(\hat{w})$, sample $z' \in \cS$ to unlearn.}
\KwResult{model $w'$ with $z'$ removed.}
1. \For{$j \in \{1,\ldots, m\}$ perturbations}{
    $\xi^j \sim N(0,\sigma^2)$ \\
    $z^{\prime,j} = z^\prime + \xi^j$ \\
    $l^j, a^j = f(z^{\prime,j})$ \\
}
2. Compute $P* = \mbox{L-FOCI}(l^J, a^J)$. \\
3. Compute $\nabla^2_{P} F(\hat{w}, z^\prime)$ via finite differences. \\
4. Update:
\begin{align}
    H_P' &= \frac{1}{n-1}\left(n \nabla^2_{P} F(\hat{w}) - \nabla^2_{P} f(\hat{w}, z')\right) \\
    w_{P}' &= \hat{w}_{P} + \frac{1}{n-1}H_{P}'^{-1} \nabla f(\hat{w}, z')_{P} \\
    w'_{\Theta\setminus P} &= \hat{w}_{\Theta\setminus P}
\end{align}
 \caption{\label{alg:blockunlearn} Unlearning via Conditional Dependence Block Selection}
\end{algorithm}

\noindent\textbf{Computational Gains.}
A direct observation is that now we are doing sampling, which adds a linear computational load.
However, directly updating all parameters requires $O(d^3)$ computation due to matrix inversion, while this procedure requires $O(md + dm\log m + p^3)$, for the forward passes, FOCI algorithm, and subsequent subsetted matrix inversion. For any reasonable setting, we have $p \ll d$, and so this clearly offers significant practical advantages.

\subsection{Theoretical Analysis}
By definition, any neural network as described above is actually a Markov Chain: we know that the output of a layer is conditionally independent of the penultimate one given the previous one, and clearly a change in one layer will propagate forward through the rest of the network.
However, when trained for a task with a large number of samples, the influence or ``memory" of the network with respect to a specific sample may not be clear.
While the output of the layers may follow a Markov Chain, the parameters in the layers themselves do not, and their influence on a sample through the forward pass may be highly dependent or correlated.
Practically, we would hope that unlearning samples at convergence does not cause too much damage to the model's performance on the rest of the input samples.
Following traditional unlearning analysis, we can bound the \textit{residual gradient norm} to relieve this tension.
\begin{lemma}
The gap between the gradient residual norm of the FOCI Unlearning update in Algorithm~\ref{alg:blockunlearn} and a full unlearning update via \eqref{eq:sekhariunlearn},
\begin{align}
||\nabla F(w^-_{Foci},D')||_2 - ||\nabla F(w^-_{Full},D')||_2
\end{align}
shrinks as $O(1/n^2)$.
\end{lemma}
\begin{proof}
The full proof is in the appendix. Main idea: Because we only update a subset of parameters, the gradients for the remainder should not change too much. Any change to a selected layer only propagates to other layers by $1/n$, and a Taylor expansion about the new activation for that layer gives the result.
\end{proof}

\noindent\textbf{How L-CODEC achieves acceleration for Unlearning?}
Sampling with weights proportional to the Lipschitz constant of individual filters/layers is an established approach in optimization, see \cite{gorbunov2020unified}. We argue that L-CODEC computes an approximation to optimal sampling probabilities.
Under a mild assumption that the sampling probabilities have \emph{full} support, it turns out that correctness of our approximate (layer/filter selection) procedure can be guaranteed for unlearning purposes using recently developed optimization tools, see \cite{gower2019sgd}. By adapting results from \cite{gorbunov2020unified}, we can show the following, summarizing the main result of our slice-based unlearning procedure.
\begin{theorem}
Assume that layer-wise sampling probabilities are nonzero. Given unlearning parameters $\epsilon,\delta$, the unlearning procedure in Alg~ \ref{alg:blockunlearn} is $(\epsilon',\delta')-$forgetting where $\epsilon'>\epsilon,\delta'>\delta$ represent an arbitrary  precision (hyperparameter) required for unlearning. Moreover, iteratively applying our algorithm converges exponentially fast (in expectation) w.r.t. the precision gap, that is, takes (at most) $O(\log\frac{1}{\mathbf{g_{\epsilon}}}\log\frac{1}{\mathbf{g_{\delta}}})$ iterations to output such a solution where  $\mathbf{g_{\epsilon}} = \epsilon'-\epsilon>0,\mathbf{g_{\delta}}=\delta'-\delta>0$ are gap parameters.
\end{theorem}
Our result differs from Nesterov's acceleration: we do not use previous iterates in a momentum or ODE-like fashion; rather, here we are closer to primal-dual algorithms where knowing nonzero coordinates at the dual optimal solution  
can be used to accelerate primal convergence, see \cite{diakonikolas2019approximate}. Moreover, since our approach is {\em randomized}, the dynamics can be better modeled using the SDE framework for unlearning purposes, as in \cite{simsekli2020fractional}. 
Here, we do not  compute anything extra, although it is feasible for future extensions.
\begin{remark}
Our approach to estimate the Lipschitz constant is different from \cite{fazlyab2019efficient} where an SDP must be solved -- quite infeasible for unlearning applications. Our approach can be interpreted as solving a simplified form of the SDP proposed there, when appropriate regularity conditions on the feasible set of the SDP are satisfied. 
\end{remark}

\noindent\textbf{A note on convexity.}
Existing methods for guaranteeing removal and performance depend on models being convex. Practical deep learning applications however involve highly nonconvex functions.
The intuitions of unlearning for convex problems {\bf directly apply to nonconvex unlearning} with one more technical assumption: minimizers of the learning problem satisfy Second Order Sufficiency (SOS) conditions. 
SOS guarantees that $ \nabla^2\hat{F}(\hat{w})$, $\hat{H}$ in eq (7) of [28] are PSD, and that the update (8) is an {\em ascent} direction w.r.t. the loss function on $U$, making unlearning possible. Guarantees for nonconvex unlearning involve explicitly characterizing a subset of SOS points (so-called  ``basin of attraction" of population loss), i.e., which points gradient descent can converge to, see \S1.3 in \cite{Traonmilin_2020}. 
%
So, will minimizers from first order methods satisfy SOS conditions? Generally, this is not true, e.g., when the Hessian is indefinite, $\hat{H}\not\succeq 0$, the update itself may not be an ascent direction w.r.t. negative of the loss. Here, standard Hessian modification schemes are applicable \cite{wright1999numerical}, subsequently using the Newton's step in \cite{sekhari2021remember} with a diagonally modified Hessian.

We fix weight decay during training, acting as $\ell_2$ regularization and giving us an approximate $\lambda$-strong convexity. We also take advantage of this property to smooth our Hessian prior to inversion, intuitively extending the natural linearization about a strongly-convex function. Interestingly, this exactly matches a key conclusion from \cite{basu2021influence}: weight-decay heavily affects the quality of 
the measured influence,  
consistent with our nonconvexity discussion.
\noindent{\bf Implementation Details.}
As we only need a subset of the Hessian, we compute the finite difference among the parameters within the blocks selected.
For large models, even subsets of model parameters may lead to large Hessian computations, so we move parameters as needed to the CPU for parameter updates. Pairwise distance computations for CI testing via nearest neighbor are carried out on the GPU \cite{yoso-zhanpeng}.
Our code although not explicitly optimized achieves reasonable run-time for unlearning for deep models, e.g., one unlearning step for person re-identification task on a ResNet50 model with roughly 24M parameters takes about 3 minutes.  




\section{L-FOCI in Generic ML Settings}
\begin{figure}
    \centering
    \includegraphics[width=0.98\columnwidth,trim={0, 11.6cm, 0, 0},clip]{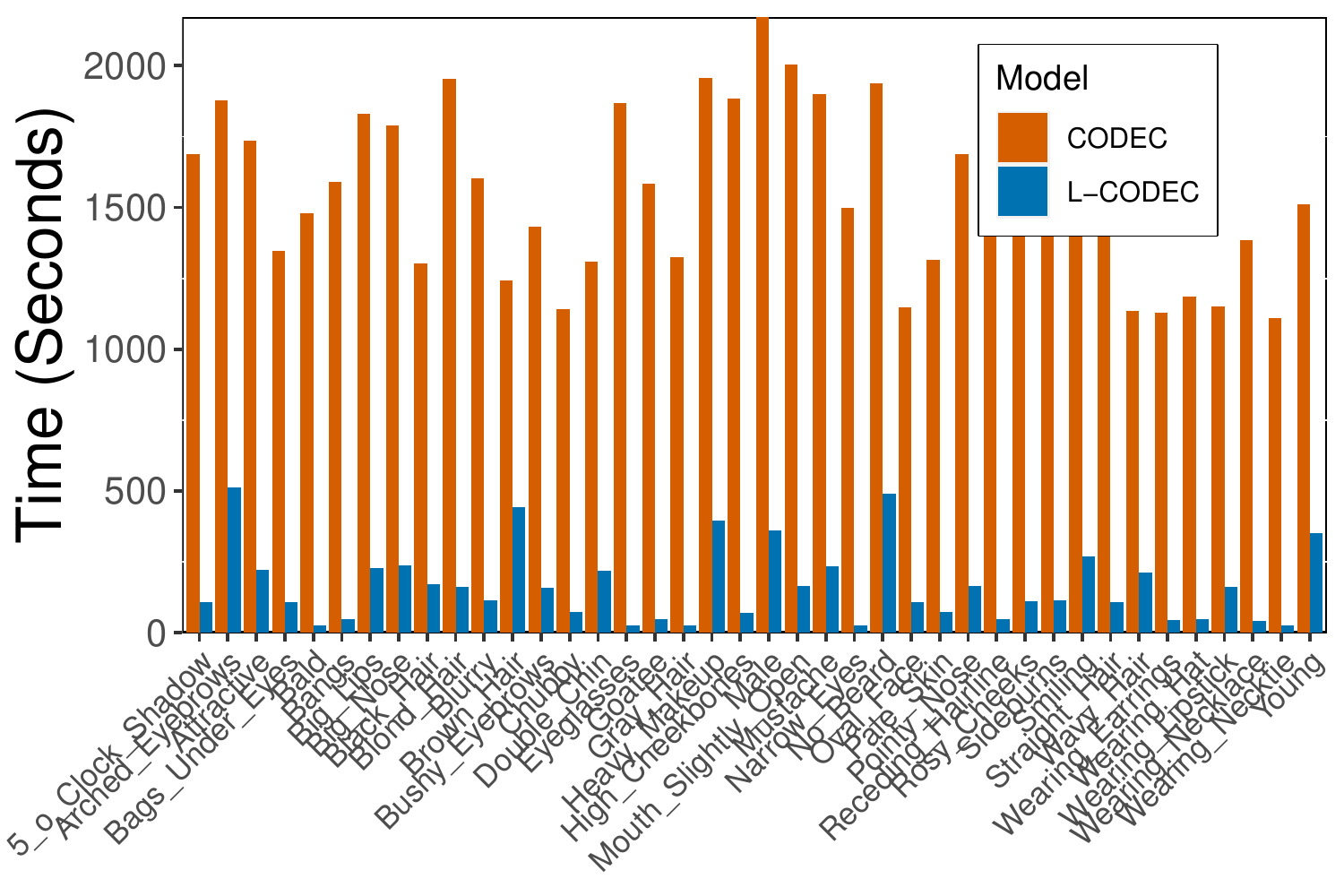}
    \vspace{-8pt}
    \caption{\label{fig:speed_hist} L-CODEC vs CODEC run time comparison for identifying sufficient subsets for each CelebA attribute separately (pairs of columns, details in supplement).}
\end{figure}
We begin with understanding the value of L-CODEC and L-FOCI for Markov Blanket Identification and progress to applications
in typical unlearning tasks involving large neural networks previously infeasible with existing scrubbing tools.
See appendix for additional details. 

\noindent\textbf{L-CODEC Evaluation.}
To assess speedup gained in the discrete setting when running L-CODEC,
we construct the Markov Blanket for specific attributes provided as side information with the CelebA dataset. Fig.~\ref{fig:speed_hist} shows the wall-clock times for Markov Blanket Selection via FOCI and L-FOCI for each attribute.

\noindent\textbf{Markov Blanket Identification.} We replicate the experimental setup in Sec 5.3 of \cite{bullseye}, where a high dimensional distribution over a ground truth graph is generated, and feature mappings are used to reduce the dimension and map to a latent space. Table~\ref{tab:bullseye} summarizes subset identification efficacy and runtime. Replacing conditional mutual information (CMI) with L-CODEC, we see a clear improvement in both runtime and Markov Blanket identification over the raw data, and comparable results in the latent feature space. Using L-FOCI directly in the feature space, we identify an additional spurious feature not part of the Markov Blanket, but runtime is significantly faster.
\begin{table}
    \centering
    \resizebox{!}{45pt}{%
    \begin{tabular}[b]{l|ccr|ccr}
        \hline\hline
        & \multicolumn{3}{c|}{Raw Data} & \multicolumn{3}{c}{Feature Maps} \\
        Method & TPR & FPR & Time (s) & TPR & FPR & Time (s) \\
        \hline
        \cite{bullseye} & 0.75 & 0.50 & 5124.22 & \textbf{0.875} & \textbf{0.00} & 516.19 \\
        L-CODEC + CIT & \textbf{1.00} & \textbf{0.50} & \textbf{402.10} & 0.75 & \textbf{0.00} & 117.29 \\
        L-CODEC + L-FOCI & \multicolumn{3}{c|}{N/A} & 0.833 & 0.50 & \textbf{0.464} \\
        \hline\hline
    \end{tabular}
    }
    \caption{\label{tab:bullseye} 3D-Bullseye Markov Blanket identification. CIT represents the model in \cite{bullseye}. Both L-CODEC and L-FOCI run much faster than recent Markov Blanket identification schemes. L-FOCI is not applicable to the multi-dimensional raw data setting.}
\end{table}

\begin{figure}[t]
    \centering
    \includegraphics[width=0.155\columnwidth,trim={0.25in 0 10.25in 0},clip]{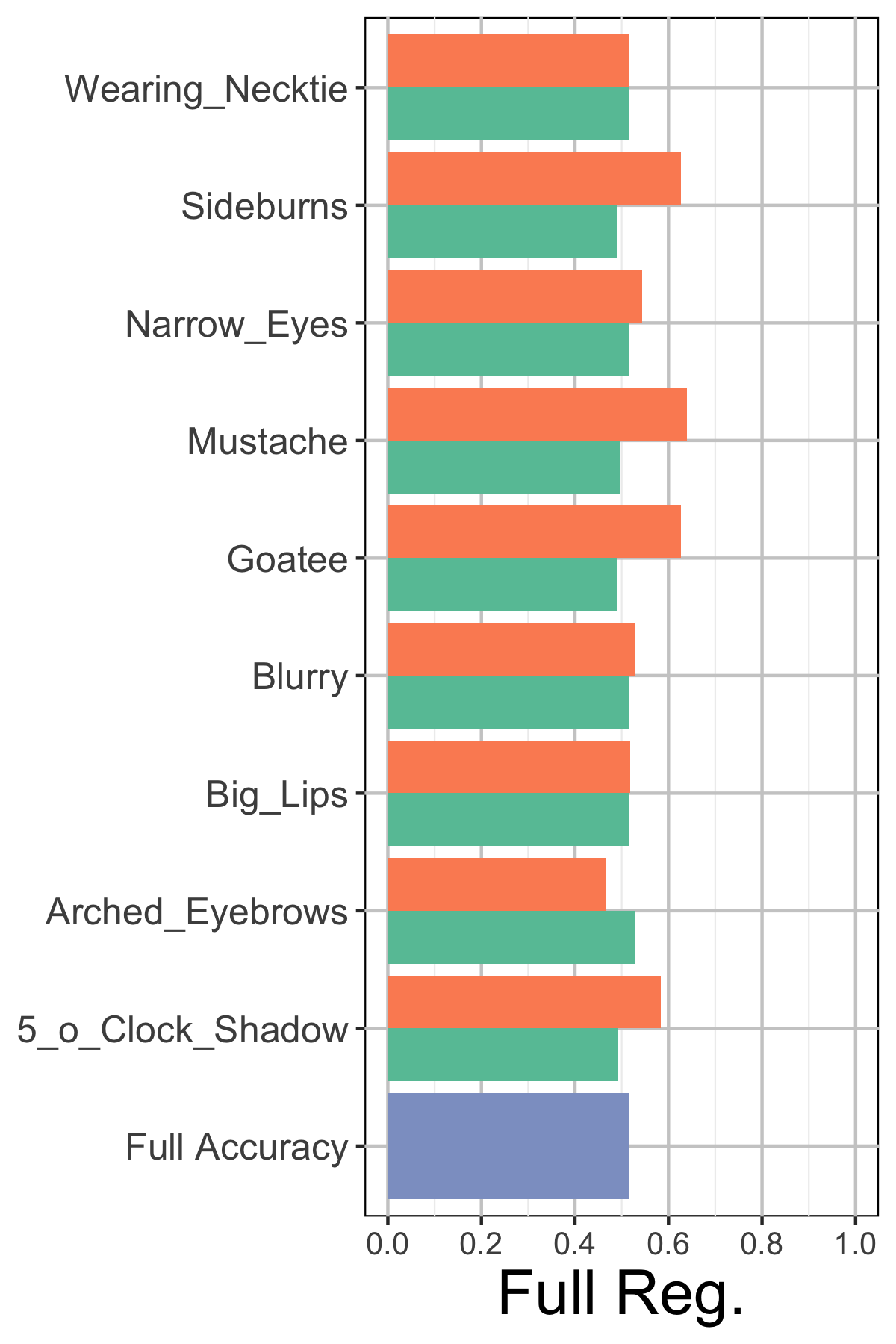}
    \includegraphics[width=0.25\columnwidth,trim={6.5in 0 0 0},clip]{figs/No_Beard_All.png}
    \includegraphics[width=0.25\columnwidth,trim={6.5in 0 0 0},clip]{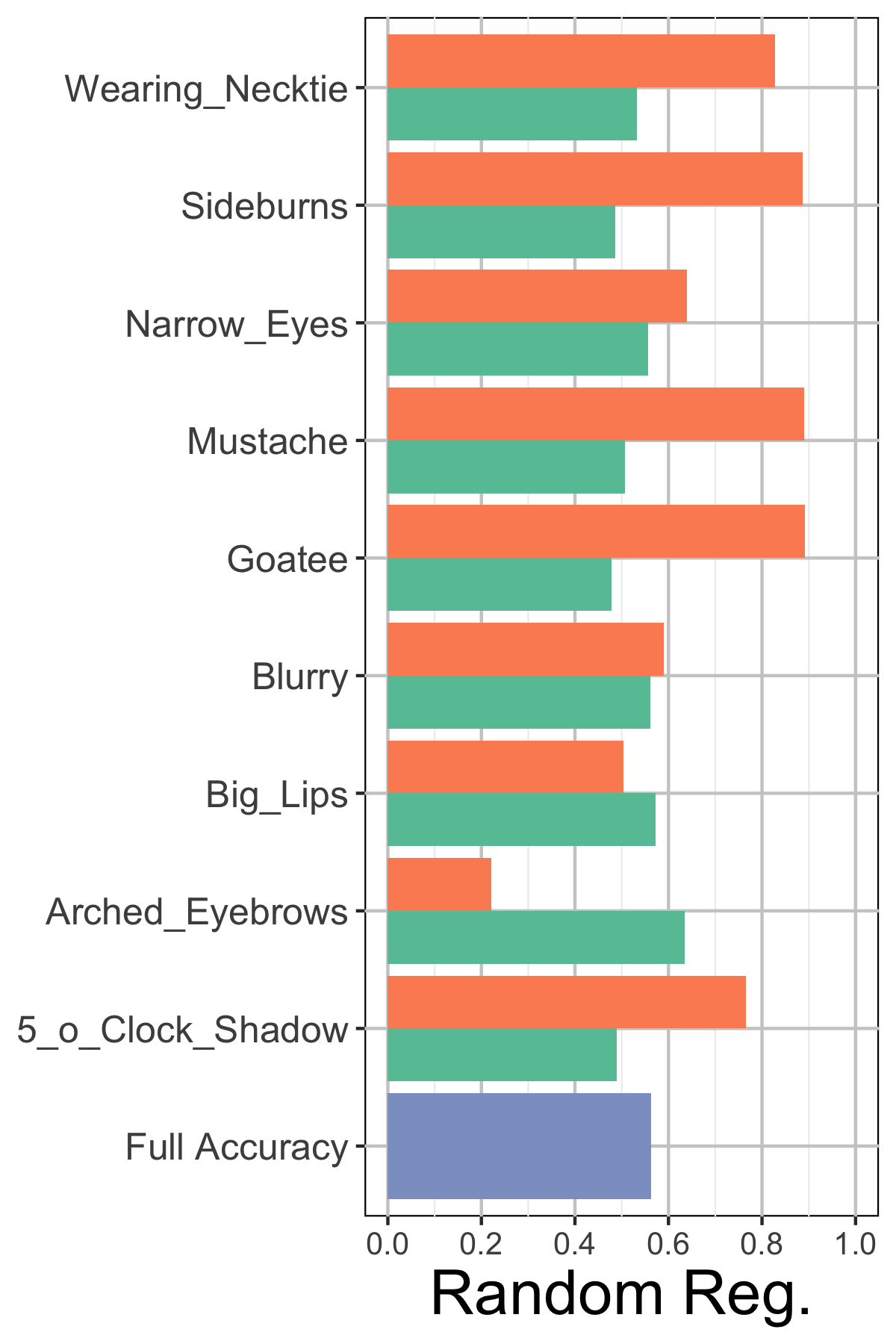}
    \includegraphics[width=0.25\columnwidth,trim={6.5in 0 0 0},clip]{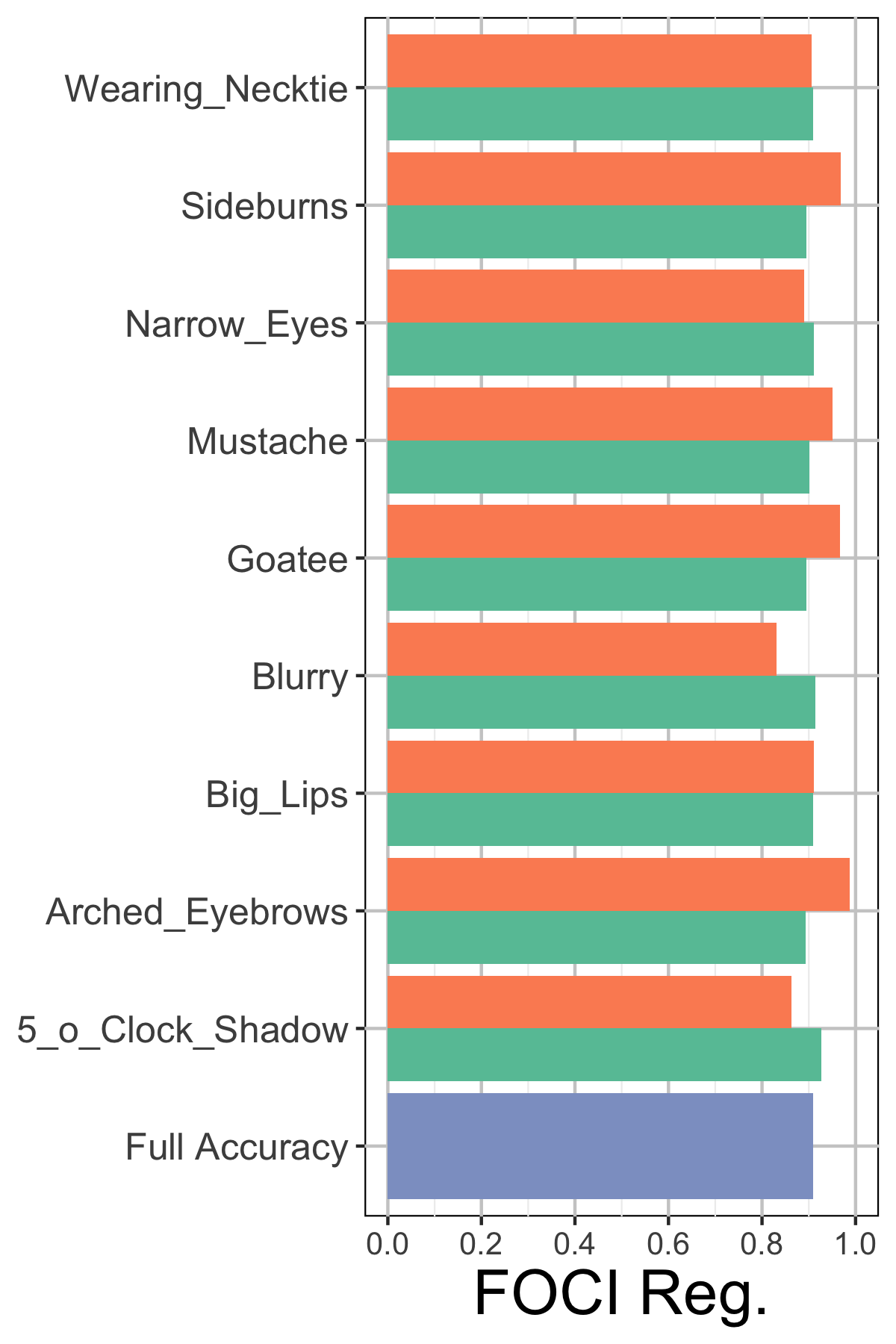}
    \vspace{-8pt}
    \caption{\label{fig:spur} Validation accuracies after training to predict ``No Beard" in the CelebA dataset. (L to R) regularization for all features, for a random subset, and via FOCI. Green indicates accuracy on the data with that feature, red, without.}
    \vspace{-15pt}
\end{figure}

\noindent\textbf{Spurious Feature Regularization.}
This Markov Blanket $(MB)$ identification scheme can be used to address spurious feature effects on traditional NN models.
A straightforward approach would be to directly add a loss term for each potentially important feature over which we would like to regularize, 
$\cL(\theta) + \sum_{S\in \cS} R_S(\theta)$.
However, with a large number of outside factors $S$, this can adversely effect training.
We instead use L-FOCI to identify the set of minimal factors that, when conditioned, make the rest conditionally independent. Then it is only necessary to include regularizers over $S \in MB(Y)$.

We evaluate a simple attribute image classification setting using the CelebA dataset. We run L-FOCI over the attributes as in our L-CODEC evaluation, and regularize using a Gradient Reversal Layer for a simple accuracy term over those attributes.
Results in Fig. \ref{fig:spur} clearly show that selection with FOCI provides the best result, maintaining high overall accuracy but also preserving high accuracy on sets of samples with/without correlated attributes.

\section{L-FOCI for Machine Unlearning}

\subsection{Compare to Full Hessian Computation}
For simple regressors, we can compute the full Hessian and compare results generated by a traditional unlearning update, our L-FOCI update, and a random selection update. To reduce variance and show the best possible random selection, we run our L-FOCI and randomly choose a set of the same size for each random selection. Fig.~\ref{fig:mnistcifar} (left) shows validation and residual accuracies for 1000 random removals from MNIST (average over 10 runs). 

\begin{figure*}[h!]
    \centering
    \includegraphics[width=0.24\textwidth]{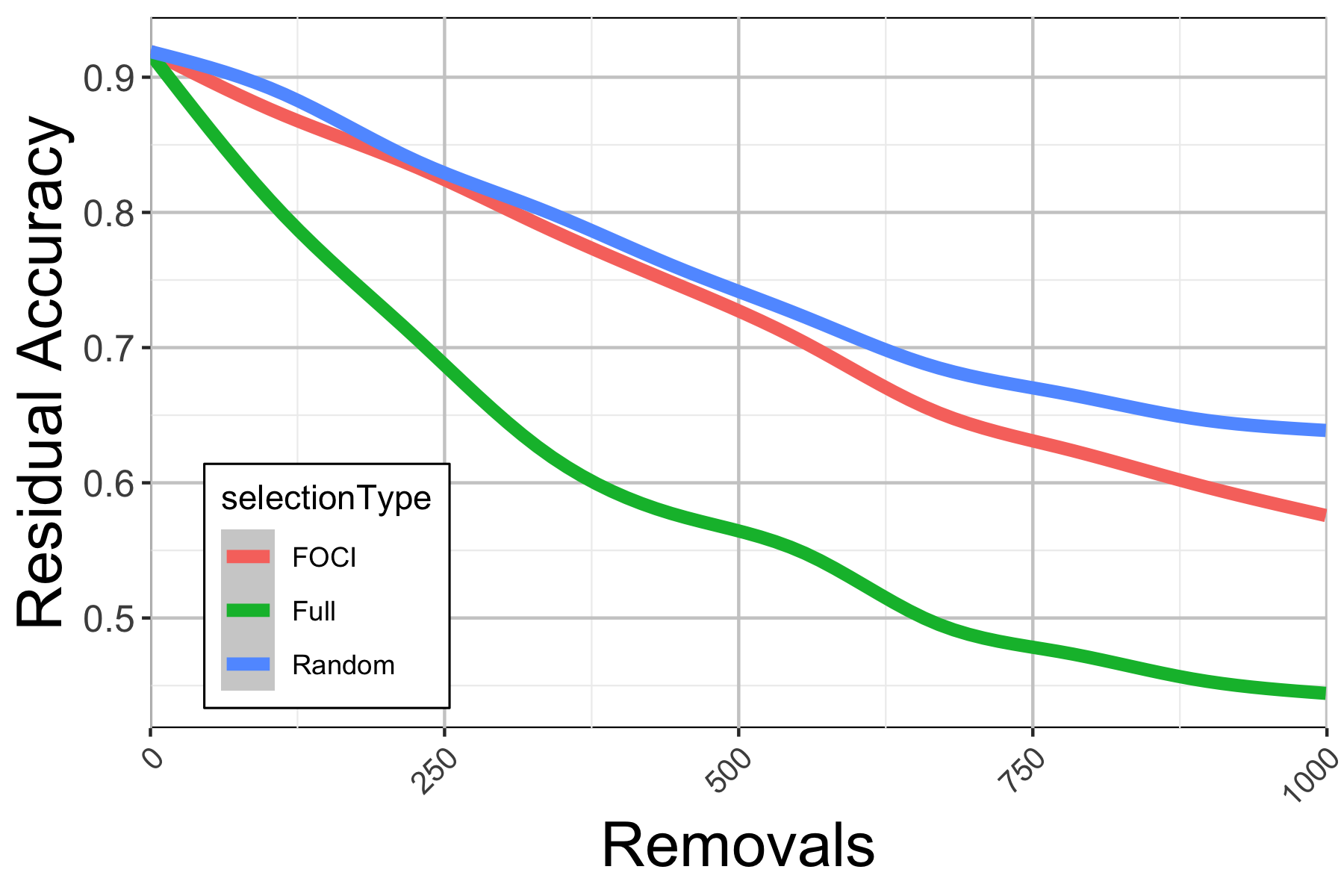}
     \includegraphics[width=0.24\textwidth]{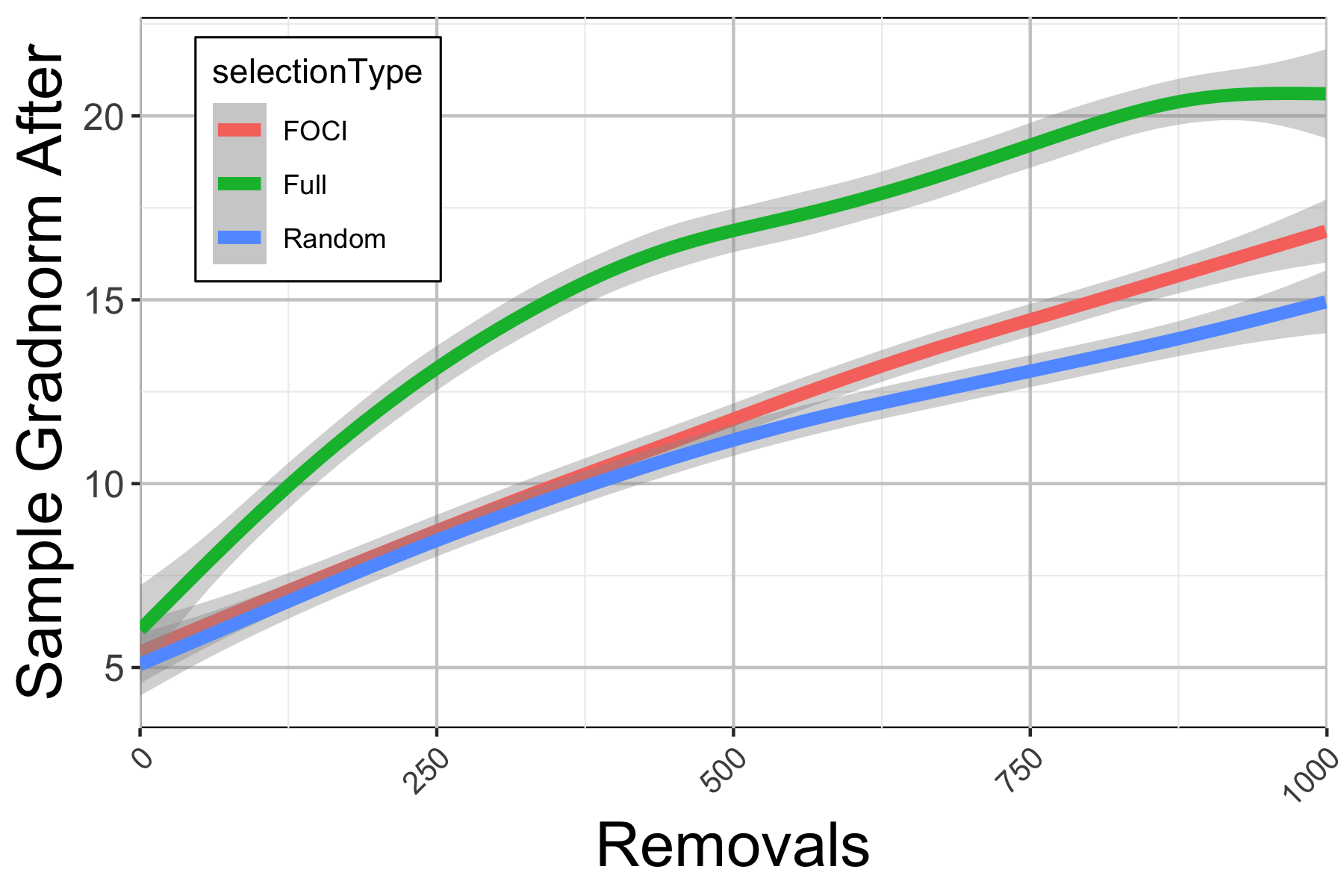}
     \vspace{-8pt}
    \includegraphics[width=0.24\textwidth]{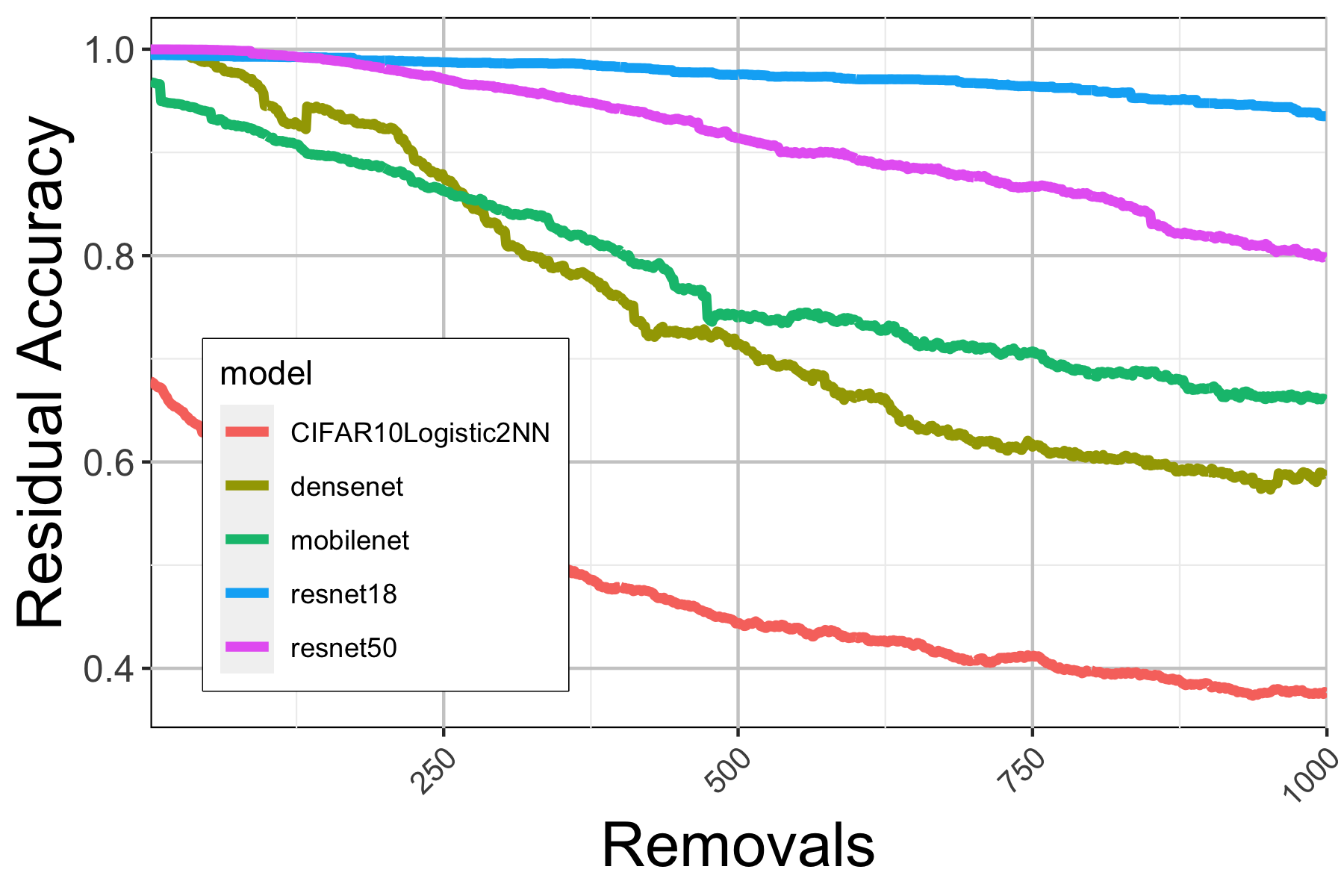}
    \includegraphics[width=0.24\textwidth]{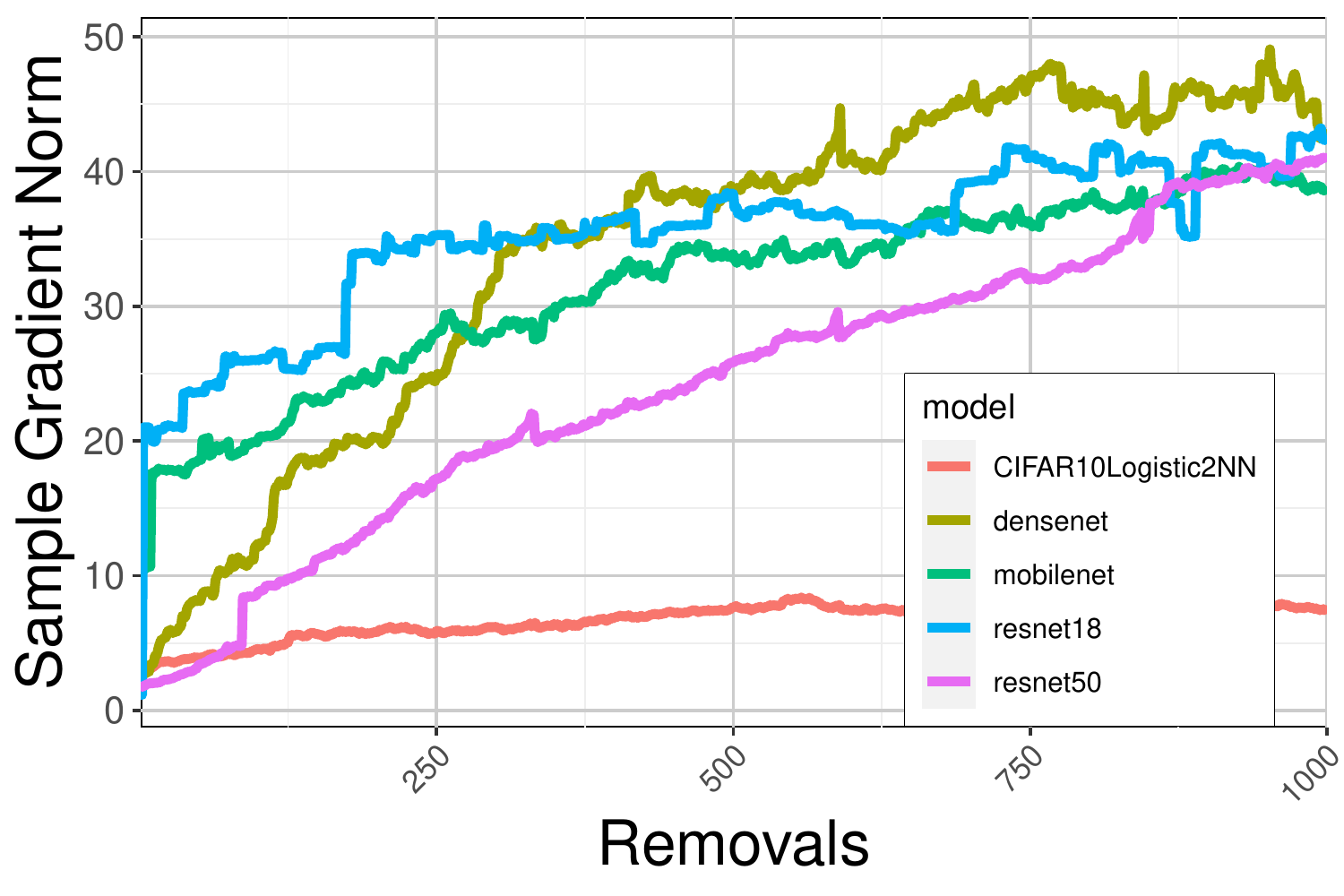}
    \caption{(Left) Residual Accuracies \& Sample Gradient Norm of removal for an MNIST Logistic Regressor. Averaged over 10 runs. (Right) Residual accuracies and sample gradient norms for various CIFAR-10 models.}
    \label{fig:mnistcifar}
\end{figure*}

\noindent\textbf{Are we selecting reasonable subsets?} A natural question is whether the subset selection via L-FOCI is any better than random, given that we are effectively taking a smaller global step. We answer this in the affirmative with a simple comparison with a random selection of size equal to the set selected by L-FOCI. Fig.~\ref{fig:mnistcifar} (left) shows that the sample gradient norm for selections made by L-FOCI are larger than those of a random selection: the subset of the model scrubbed of this specific sample has a larger impact on its final loss, and thus the gradient norm post-removal is large.



\noindent\textbf{Does the formulation scale?} We scrub random samples from various CIFAR-10 models, and evaluate  performance for the same set of hyperparameters. When the models are larger than logistic regression, it is infeasible to estimate the full Hessians, so 
we {\em must} use our L-FOCI selection update. Fig.~\ref{fig:mnistcifar} (right) shows removal performance over many typical models with varying sizes. Models that have higher base accuracies tend to support more removals before performance drops. This matches results for differentially private models: models that generalize well may not have overfit and thus may already be private, allowing ``fast'' forgetting.


\noindent\textbf{Tradeoff vs Retraining.}
While our focus is the setting in which retraining is not feasible, where we can retrain we compare validation accuracies as a function of number of removals. Using a subset of MNIST, we train to convergence and iteratively remove samples using our construction, retraining fully at each step for comparison. With 1000 training samples from each class 
and reasonable settings of privacy parameters ($\epsilon=0.1,\delta=0.01$),
we support a large percentage of removals until validation accuracy drops more than a few percent, see Fig.~\ref{fig:retrain}.
\begin{figure}
    \centering
\includegraphics[width=0.49\columnwidth]{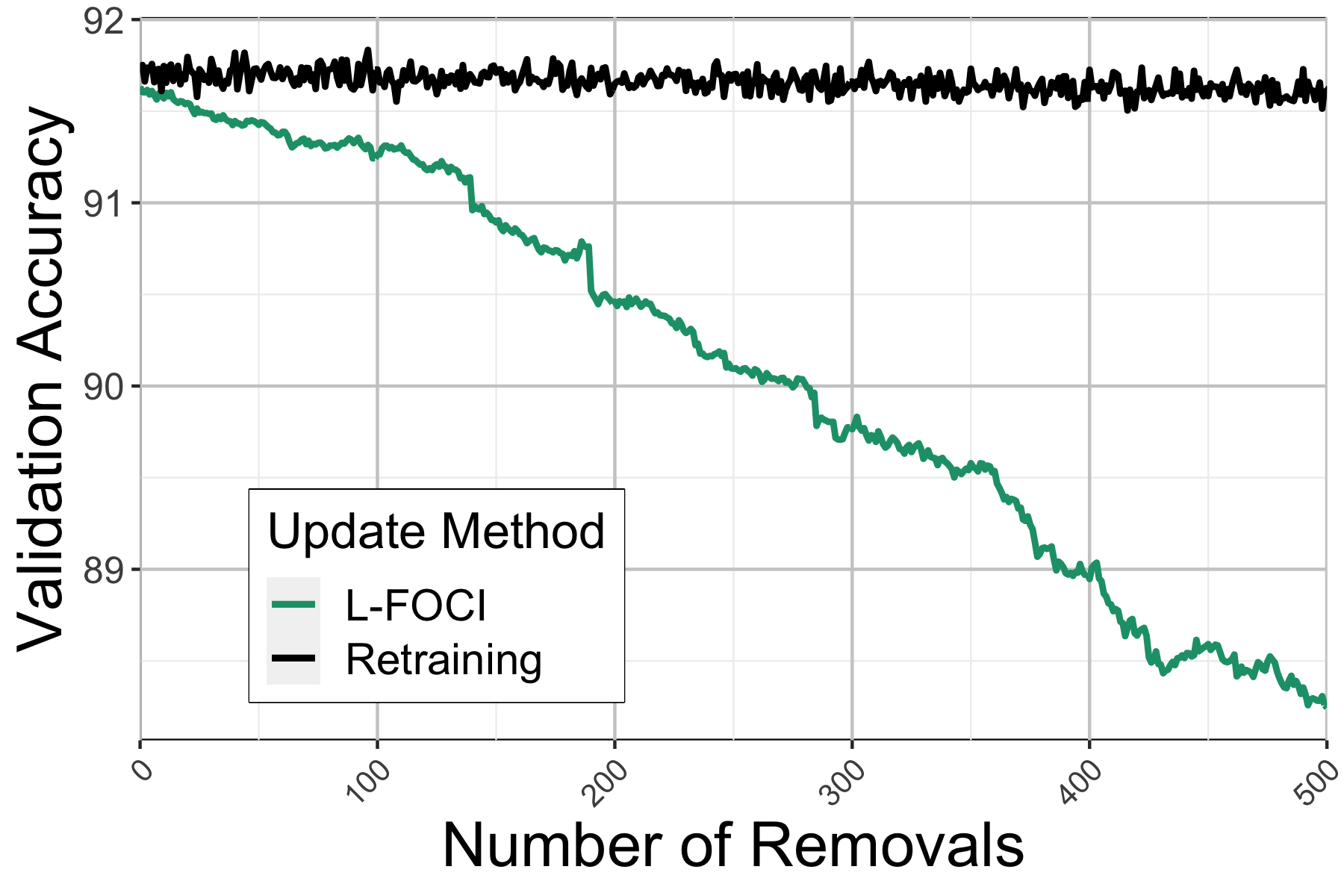}
\includegraphics[width=0.49\columnwidth]{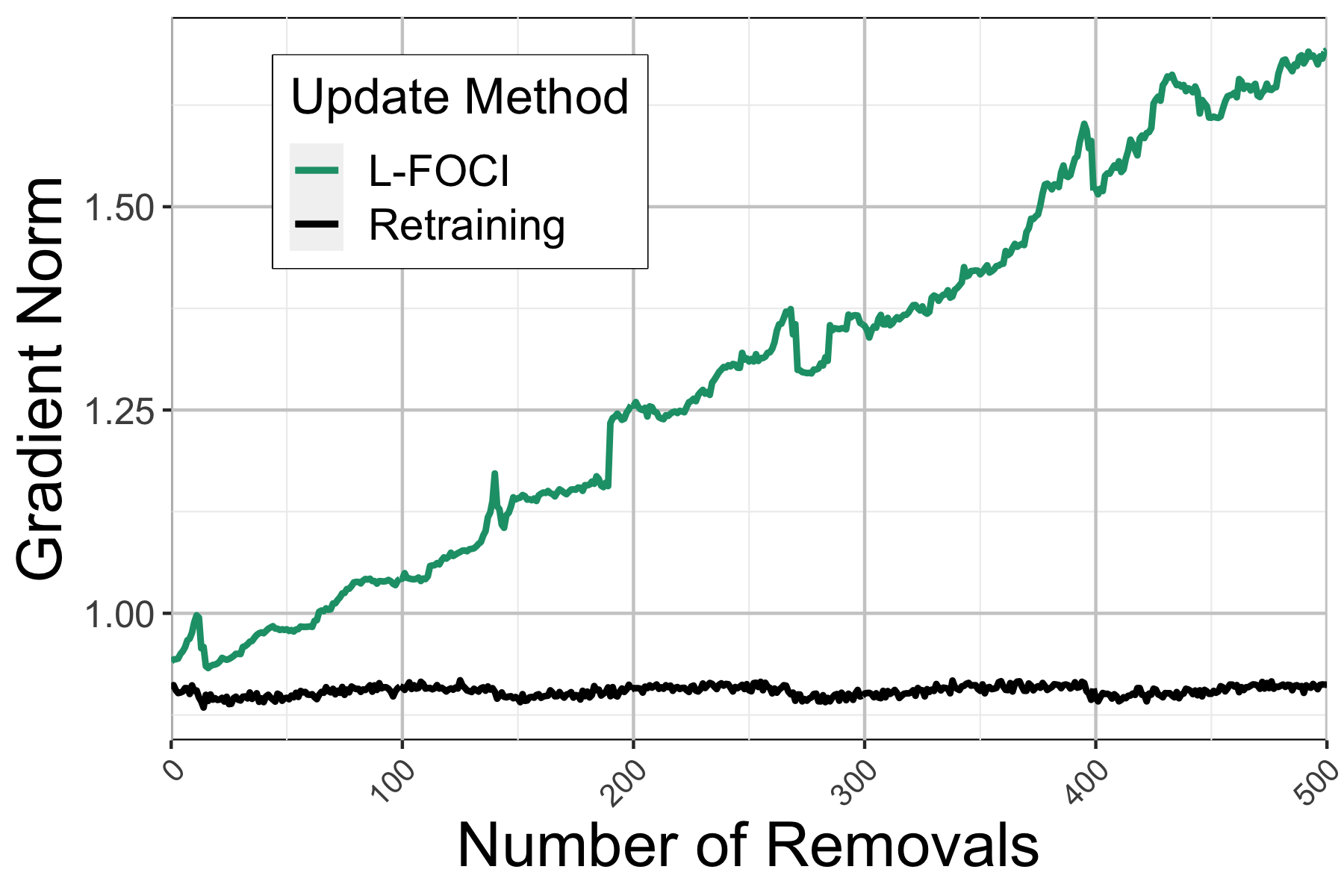}
    \caption{MNIST Retraining comparison averaged over 8 runs. Validation accuracies and residual gradient norms.}
    \label{fig:retrain}
\end{figure}

\subsection{Removal in NLP models}
We now scrub samples from transformer based models using  LEDGAR \cite{tuggener2020ledgar}, a multilabel corpus of legal provisions in contracts. We use the prototypical subset which contains $110156$ provisions pertaining to $13$ most commonly used labels based on frequency. Our model is a fine-tuned DistilBERT \cite{sanh2019distilbert} and uses the $[CLS]$ token as an input to the classification head. 
Table. \ref{tab:nlp} shows results of scrubbing the provisions from two different classes; \textit{Governing Laws} and \textit{Terminations} which have the highest/lowest support in the test set. As expected with increasing $\epsilon$, i.e., lower privacy guarantees, we can support more number of removals based on the Micro F1 score of the overall model. 
The Micro F1 scores, for the removed class fall off rapidly, while the change in overall scores is more gradual. 



\subsection{Removal from Pretrained Models}
The above settings show settings where a sample from one specific source may be removed. A more direct application of unlearning 
is completely removing samples from a specific class; a compelling use case is face recognition.

We utilize the VGGFace dataset and model, pretrained from the original work in \cite{huang2008labeled,Parkhi15}. The model uses a total of approximately 1 million images to predict the identity of 2622 celebrities in the dataset. Using a reconstructed subset of 100 images from each person, we first fine-tune the model on this subset for 5 epochs, and use the resultant models as estimates of the Hessian. 
In this setting, the VGGFace model is very large, including a linear layer of size $25088 \times 4096$. Selecting even a few slices from this layer results in a Hessian matrix unable to fit in typical memory. For this reason, we run a ``cheap" version of L-FOCI: we select only one slice that results in the largest conditional dependence on the output loss.
%

\stepcounter{table}
\begin{figure}
\begin{subfigure}{0.45\columnwidth}\centering
    \includegraphics[width=0.98\columnwidth]{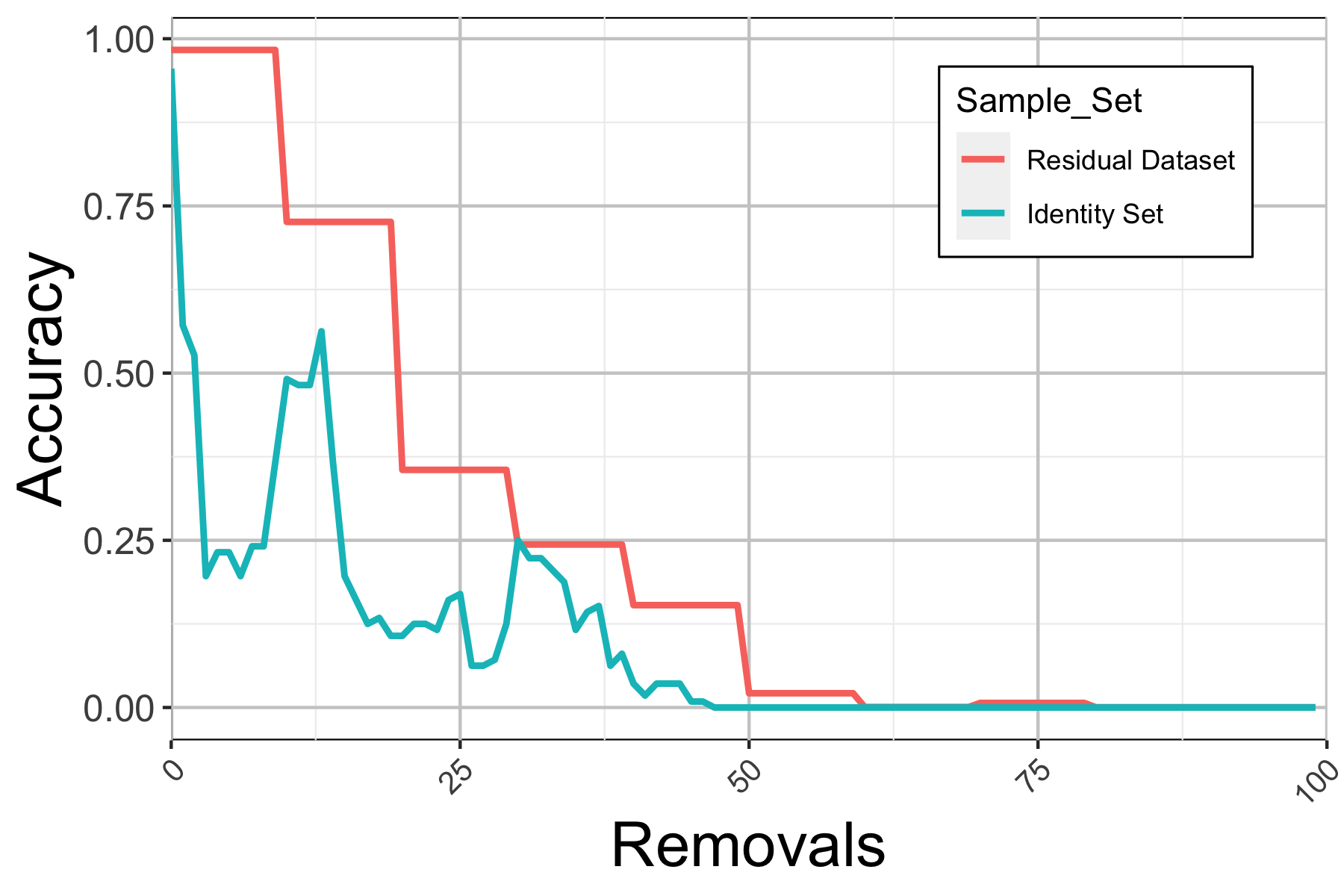}%
    \phantomsubcaption
    \label{fig:vgg}
\end{subfigure}
\begin{subfigure}{0.495\columnwidth}
    \centering
    \resizebox{!}{50pt}{%
    \begin{tabular}[b]{l|ccc}
        \hline\hline
        & \multicolumn{2}{c}{\# Supported Removals}  \\
        $\epsilon$ & Governing Laws & Terminations \\
        \hline
        0.1 & $>$ 100 & $>$ 100 \\
        0.01 & $>$ 100 & $>$ 100 \\
        0.001 & 18 & 21 \\
        0.0005 & 6 & 7 \\
        \hline\hline
    \end{tabular}%
    }
    \phantomsubcaption
    \label{tab:nlp}
\end{subfigure}
\captionsetup{labelformat=andtable}

\caption[important image]{(Left) Scrubbed and Residual Accuracies (every 10 removals) for $\epsilon = 1e^{-5}$. The accuracy drop for the residual set is gradual up to a certain number of removals. (Right) Scrubbing transformer model for provision classification.}
\label{fig:combine}
\vspace{-5pt}
\end{figure}

Fig.~\ref{fig:vgg} show results for scrubbing consecutive images from one individual in the dataset for a strong privacy guarantee of $\epsilon=10^{-5}$. As the number of samples scrubbed increases, the performance on that class drops faster than on the residual set, exactly as desired.

\subsection{Removal from Person re-identification model}
As a natural extension to our experiments on face recognition, we evaluate unlearning of deep neural networks trained for person re-identification. Here, the task is to associate the images pertaining to a particular individual but collected in diverse camera settings, both belonging to the same camera or from multiple cameras. In our experiments, we use the Market-1501 dataset \cite{zheng2015scalable} and a Resnet50 architecture which was trained for the task. We unlearn samples belonging to a particular person, one at a time, and check the performance of the model. Experimental results are in agreement with results reported for the transformer model as well as the VGGFace model. With very small values of $\epsilon$ i.e. $0.0005$ the number of supported removals is limited to less than $10$ depending on the person id being removed. However, with a larger value of $\epsilon$, e.g., $0.1$, all potential samples can be removed without a noticeable degradation in model performance in terms of mAP scores. In Fig.~\ref{fig:reid}, we clearly see that after scrubbing a model for a particular person, its predictions for that particular individual become meaningless whereas the predictions on other classes are still possible with confidence, as desired. Additional experiments with different datasets, model architectures and other ablations for deep unlearning for person re-identification models are presented in the appendix.

\begin{figure}[!tb]
    \centering
    \includegraphics[width=0.49\columnwidth]{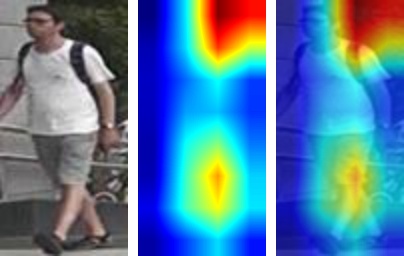}
     \includegraphics[width=0.49\columnwidth]{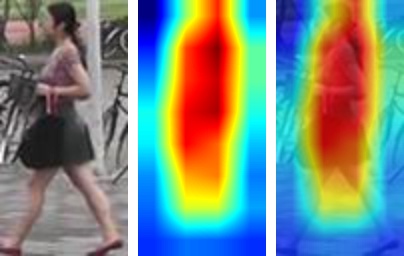}
    \caption{\label{fig:reid}Activation maps from a model scrubbed for the person on the left (right set is not scrubbed). For each triplet, from (L to R) are the original image, the activation map and its image overlay. Note the effect of scrubbing: activations change significantly for the scrubbed sample (compare column 2 to 3) whereas remain stable for the non-scrubbed sample (compare column 5 to 6).}
\end{figure}

\section{Conclusion}
Our selection scheme identifies a subset of parameters to update and significantly reduces compute requirements for standard Hessian unlearning. 
For smaller networks with a large number of removals, retraining may be effective, but when full training sets are not available or retraining is costly, unlearning in some form is needed. 
We show the ability to approximately unlearn for large models prevalent in vision, a capability that has not so far been demonstrated.  

{\bf Social Impact.} 
Indiscriminate use of personal data in training 
large AI models is ethically questionable 
and sometimes illegal. We need mechanisms to ensure that AI models operate 
within boundaries specified by society 
and legal guardrails. As opt-out laws get 
implemented, compliance on the service-provider end will entail costs. 
While our contributions cannot guarantee perfect forgetting, with additional 
validation they can become a part of a suite of methods for unlearning. 

\noindent {\bf Acknowledgments.}
This work was supported by NIH grants RF1AG059312, RF1AG062336 and RF1AG059869, NSF award CCF 1918211 and funds from the American Family Insurance Data Science Institute at UW-Madison. Sathya Ravi was supported by UIC-ICR start-up funds.


{\small
\bibliographystyle{ieee_fullname}
\bibliography{refs}
}

\clearpage
\appendix
\renewcommand{\thesection}{\Alph{section}.\arabic{section}}
\setcounter{section}{0}


\section{Theoretical Results}
\subsection{Proof of Lemma 1}
Let us take $D$ to be the training set; w.l.o.g. $z$ is the point being removed. Let the residual dataset be $D' = D \setminus {z}$.
Denote $w^-_{Full}$ as the weight parameters after doing a full Hessian update and $w^-_{Foci}$ as the weight parameters after doing a FOCI selected Hessian update.
In an ideal case, we want $(w^-_{Foci}, D')$/$(w^-_{Full}, D')$ to be as close as possible to $(w^*, D')$. Note that we consider both $(w^*, D)$ and $(w^*, D')$ to be $0$ as we don't expect model parameters to change drastically for one sample once trained to convergence. 

\begin{lemma}
The gap between the gradient residual norm of the FOCI Unlearning update in Algorithm 1 and a full unlearning update via Eq. 4 in the main paper,
\begin{align}
||\nabla \mathcal{L}(w^-_{Foci},D')||_2 - ||\nabla \mathcal{L}(w^-_{Full},D')||_2
\end{align}
shrinks as $O(1/n^2)$.
\end{lemma}
\begin{proof}
Let $w$ to be a network of many linear layers with possible activation functions; we can think of the norm as the sum of norm of gradients for each layer. Hence, for any model parameters $w$ and dataset $D$, we have:
\begin{align}
    \label{eq:res_norm}
    ||\nabla \mathcal{L}(w,D)||_2 \coloneqq \sum_{l \in L} ||\nabla \mathcal{L}(w_l, D) ||_2
\end{align}

FOCI identifies a subset $T \subset L$ slices or layers that are to be updated. 
Let $R = L \setminus T$ be the remainder of the network which is not updated.
Hence, \ref{eq:res_norm} for $(w^-_{Foci}, D')$ can be written as:
\begin{align}
    \label{eq:foci_res_norm}
    & ||\nabla \mathcal{L}(w^-_{Foci},D')||_2 \coloneqq \sum_{l \in L} ||\nabla \mathcal{L}(w^-_{Foci_l}, D') ||_2 \\
    & = \sum_{l \in T} ||\nabla \mathcal{L}(w^-_{Foci_l}, D') ||_2 + 
     \sum_{l \in R} ||\nabla \mathcal{L}(w^-_{Foci_l}, D') ||_2 \\
    & = \sum_{l \in T} ||\nabla \mathcal{L}(w^-_{Foci_l}, D') ||_2 + 
    \sum_{l \in R} ||\nabla \mathcal{L}(w^*_{l}, D') ||_2
\end{align}
The last line follows from the fact that layers in $R$ are not updated.

We will next show how for the remainder of the dataset $D'$, the changes in $T$ propagate minimally when there are a large number of data points, $n$ in the training set.

W.L.O.G. assume that we have a $3$ layer network with the form:
\begin{align}
    (L_3(L_2(L_1(x)))
\end{align}
For the point being removed $z := (x,y)$; let $L_2$ be the intermediate layer which is selected for update by FOCI.
Before the update, activations out of $L_2$ are of the form $a_2 = L_2(L_1(x)) = L_2(a_1)$.
After the update, activations out of $L_2$ can be written as:
\begin{align}
    a_2' = L_2'(L_1(x)) &= L_2'(a_1) \\
    & = w_2' a_1 \\
    & = (w_2 + \delta_{w_2}) a_1 \\
    & = w_2 a_1 + \delta_{w_2} a_1 \\
    & = a_2 + \delta_{w_2} a_1
\end{align}
The Second line follows because $L_1$ isn't updated.
For the following layer $L_3$, we have $a_3 = L_3(a_2)$ before the update. After,
\begin{align}
    a_3' &= L_3(a_2') \\
    &= L_3(a_2 + \delta_{w_2} a_1) \\
    &= L_3(a_2) + \nabla L_3 (a_2) \delta_{w_2} a_1 + \mathcal{O}((\delta_{w_2} a_1)^2) \\
    &= L_3(a_2) + 0 + \mathcal{O}((\delta_{w_2} a_1)^2) 
\end{align}
The first-order term goes to zero, as $L_3$ has not been updated and we assume full model convergence.

For the \cite{sekhari2021remember} update.
\begin{align}
    \delta_{w_2} = \frac{1}{(n-1)}(\hat{H}^{-1})\sum_{z \in \{(x_k, y_k)\}} \nabla f(\hat{w}, z)
\end{align}
Hence, $\delta_{w_2}^2 \propto \frac{1}{n^2}$. Therefore, for large values of $n$, the third term in the equation above approaches $0$. So, $a_3' = L_3(a_2)$. This shows that propagation is minimal. Similar arguments regarding null space for over-parameterized deep networks have been mentioned in \cite{golatkar2020forgetting}. 

Now, looking back at the residual gradient norm, we have:
\begin{align}
    \label{eq:foci_res_norm_simplified}
    ||\nabla \mathcal{L}(w^-_{Foci},D')||_2 &= \sum_{l \in T} ||\nabla \mathcal{L}(w^-_{Foci_l}, D') ||_2 + \\
    & \sum_{l \in R} ||\nabla \mathcal{L}(w^*_{l}, D') ||_2
\end{align}
Based on the above argument of minimal propagation, the second term above goes to $0$ for layers/slices in $R$.
Therefore,
\begin{align}
    \label{eq:foci_res}
    ||\nabla \mathcal{L}(w^-_{Foci},D')||_2 = \sum_{l \in T} ||\nabla \mathcal{L}(w^-_{l}, D') ||_2 
\end{align}
and as such the gap between this and the full update is only the difference on the set $R$, shrinking as $O(1/n^2)$.
\end{proof}

\subsection{Proof of Theorem 1}
\begin{theorem}
Assume that layer-wise sampling probabilities are nonzero. Given (user specified) unlearning parameters $\epsilon,\delta$, the unlearning procedure in Algorithm 1 in the main paper is $(\epsilon',\delta')-$forgetting where $\epsilon'>\epsilon,\delta'>\delta$ represent an arbitrary  precision (hyperparameter) required for unlearning. Moreover, iteratively applying our algorithm converges exponentially fast (in expectation) with respect to the precision gap, that is, takes (at most) $O(\log\frac{1}{\mathbf{g_{\epsilon}}}\log\frac{1}{\mathbf{g_{\delta}}})$ iterations to output such a solution where  $\mathbf{g_{\epsilon}} = \epsilon'-\epsilon>0,\mathbf{g_{\delta}}=\delta'-\delta>0$ are gap parameters.
\end{theorem}
\begin{proof}
Our proof strategy is to show that our update step in Algorithm 1 is a specific form of Randomized Block Coordinate Descent (R-BCD) method. Then, we simply apply existing convergence rates of RBCD for general smooth minimization problems.  In particular, our method can be seen as an extension of SEGA method in Corallary A.7. \cite{gorbunov2020unified} where the descent direction is provided by using inexact inverse hessian metric \cite{loizou2020convergence}.  The key difference in our setup is that the sampling probabilities are computed using the CODEC procedure instead of the random sampling at each step. We make the following three observation in our setup  that immediately asserts correctness of the procedure. 

First, by our construction in equation (11) in the main paper, the sampling probabilities have full support. That is, the probability of selecting a particular weight in the neural network is strictly positive since $\xi\sim\mathcal{N}(0,\sigma^2), \sigma>0$ is a continuous distribution which has unbounded support. Second, the overall rate of speed of convergence depends on the condition number of the (fixed) Hessian at the optimal solution since exact $(\epsilon,\delta)$ unlearning is equivalent to linear least squares problem. Third, our update step is equivalent to a projected (or sketched) primal step, see equation 13 in (ArXiv Version \cite{loizou2019convergenceArxiv}). From these observations, we can see that our overall method is equivalent to SEGA in \cite{gorbunov2020unified} or its noisy extension since we use only a small set of samples (to be unlearned) at each iteration. Consequently, we obtain the deterministic geometric rate of convergence (in expectation)  by applying Corollary A.8. where $\sigma$ in their paper corresponds to the $\epsilon'-\epsilon>0$ gap in our setup. Now, to get the probabilistic $\epsilon',\delta'$ unlearning guarantee for the solution presented by our algorithm, we use Lemma 10 in \cite{sekhari2021remember} on the solution returned, completing our proof.
\end{proof}
\section{Experimental Details}
Experiments were conducted using PyTorch 1.8 and CUDA Toolkit 10.2, run on Nvidia 2080 TIs and individual Nvidia A100s. Parallelization only occurred across runs; the attached code can be run with any CUDA/PyTorch setup with the appropriate dependencies.

\subsection{Markov Blanket Selection}
Experimental settings were taken from \cite{bullseye}, with code adapted from  \url{https://github.com/syanga/model-augmented-mutual-information}. 5000 samples were used for generating the data, and 100 trials/perumutations were conducted for the CIT testing framework. 

\subsection{L-CODEC vs CODEC Speed Results}
The full figure can be seen below, at Figure~\ref{fig:full_speed_hist}. Each pair of columns corresponds to a different attribute in the CelebA dataset as the focal variable of interest, with the rest as potential conditioning variables.
\begin{figure}[h]
    \centering
    \includegraphics[width=0.7\columnwidth]{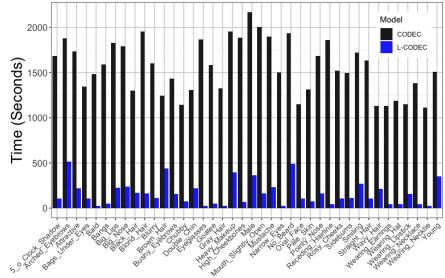}
    \caption{\label{fig:full_speed_hist} L-CODEC vs CODEC run time comparison for identifying sufficient subsets for each CelebA attribute separately (pairs of columns, details in supplement).}
\end{figure}

\subsection{MNIST Toy Results}
Training for MNIST Logistic Regressor models was run using SGD with a learning rate of 0.1, batch size of 256, and weight decay of $0.01$ for 50 epochs. 1000 perturbations were used for distribution approximation. Privacy parameters were set to $\epsilon=0.1, \delta=0.01$.
Figures and numbers in the main paper were averaged over 10 replications, for a random choice of 1000 samples to scrub.

\subsection{Retraining Comparisons}
\subsubsection{MNIST: Affects of $l_2$ Regularization and Weight Decay}
Repeating the retraining comparison in the main paper with a larger regularization, we see that the effects of removal are significantly diminished and the model can support a larger number of removals before large performance drops.
\begin{figure}
    \centering
    \includegraphics[width=0.24\columnwidth]{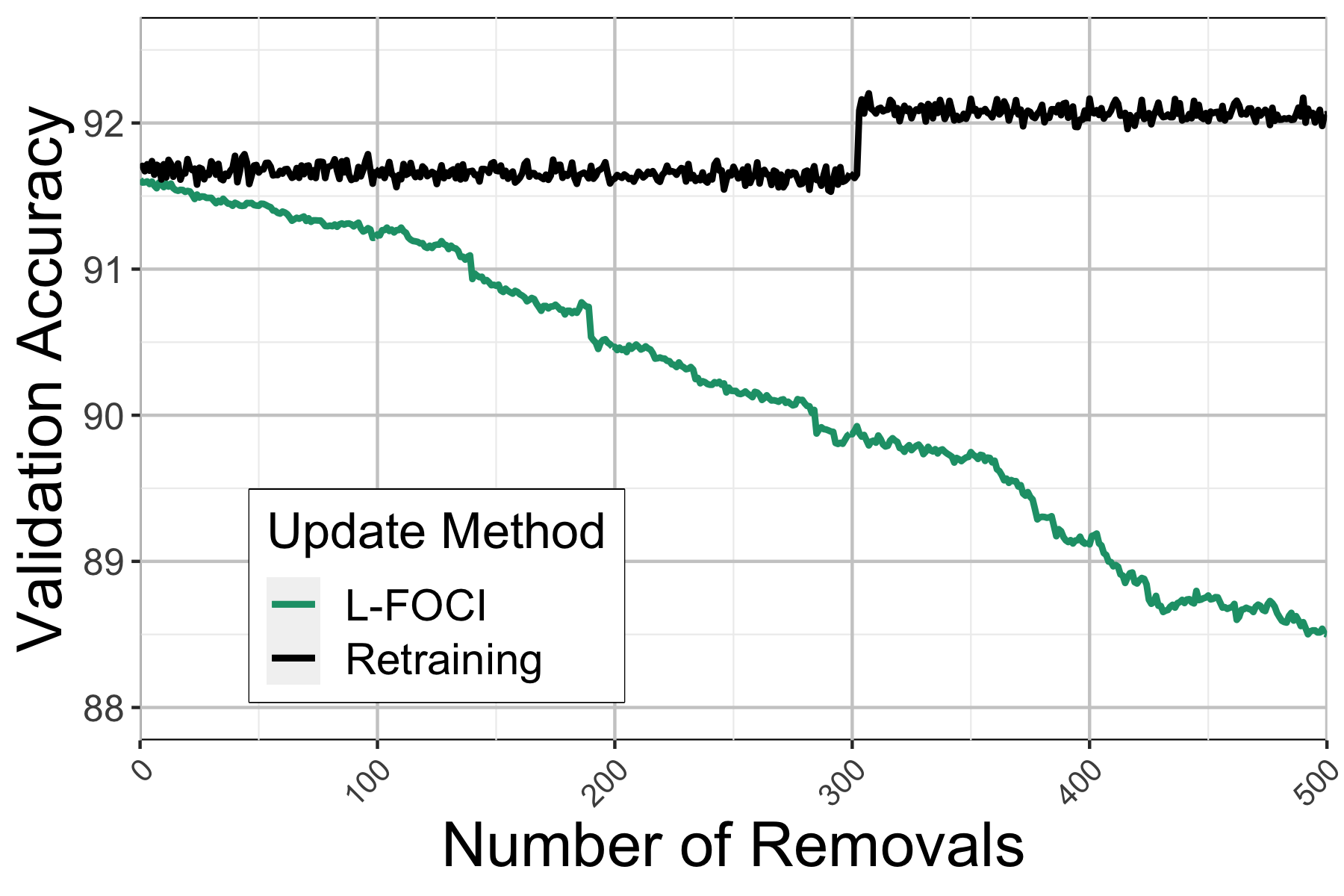}
    \includegraphics[width=0.24\columnwidth]{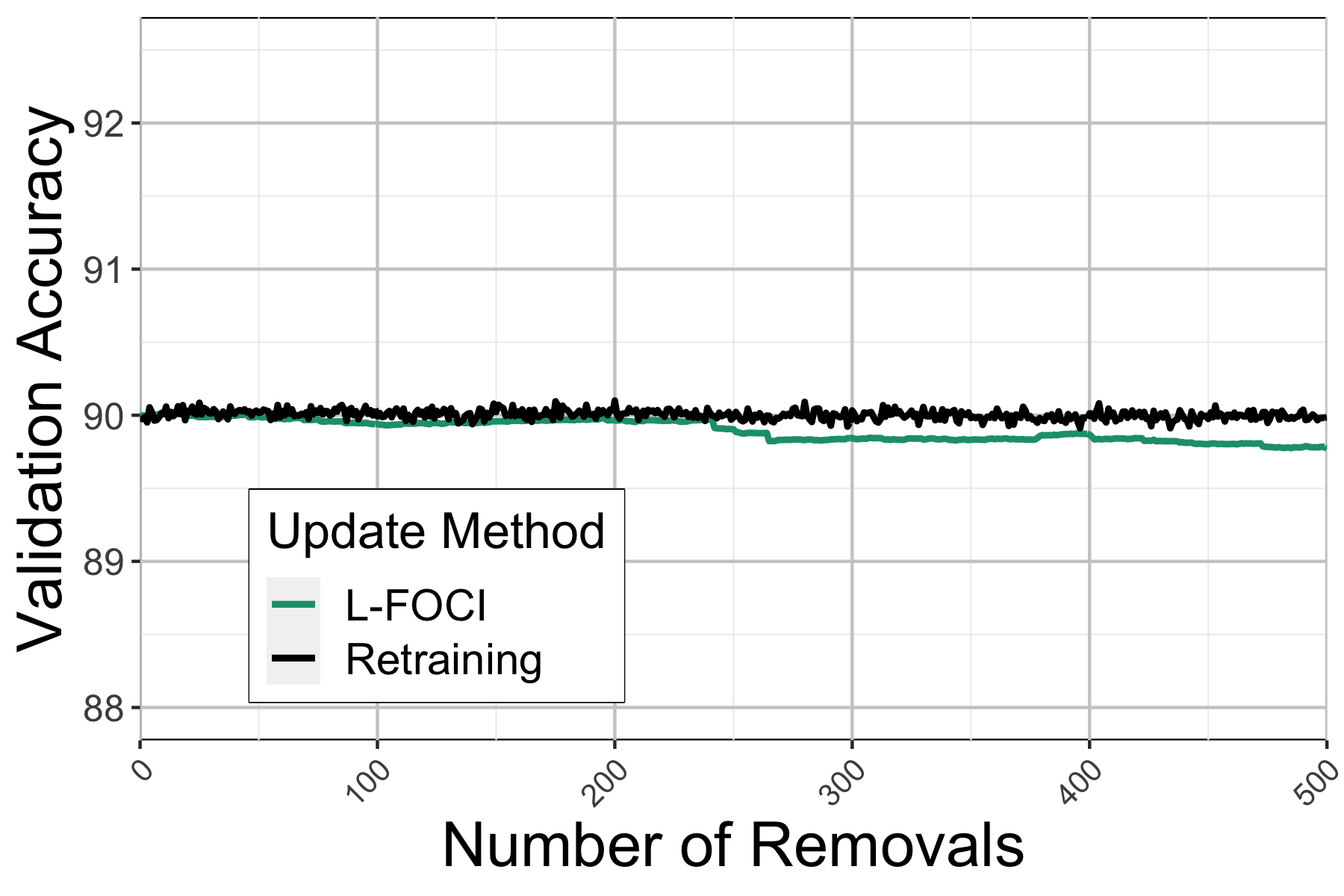}
    \includegraphics[width=0.24\columnwidth]{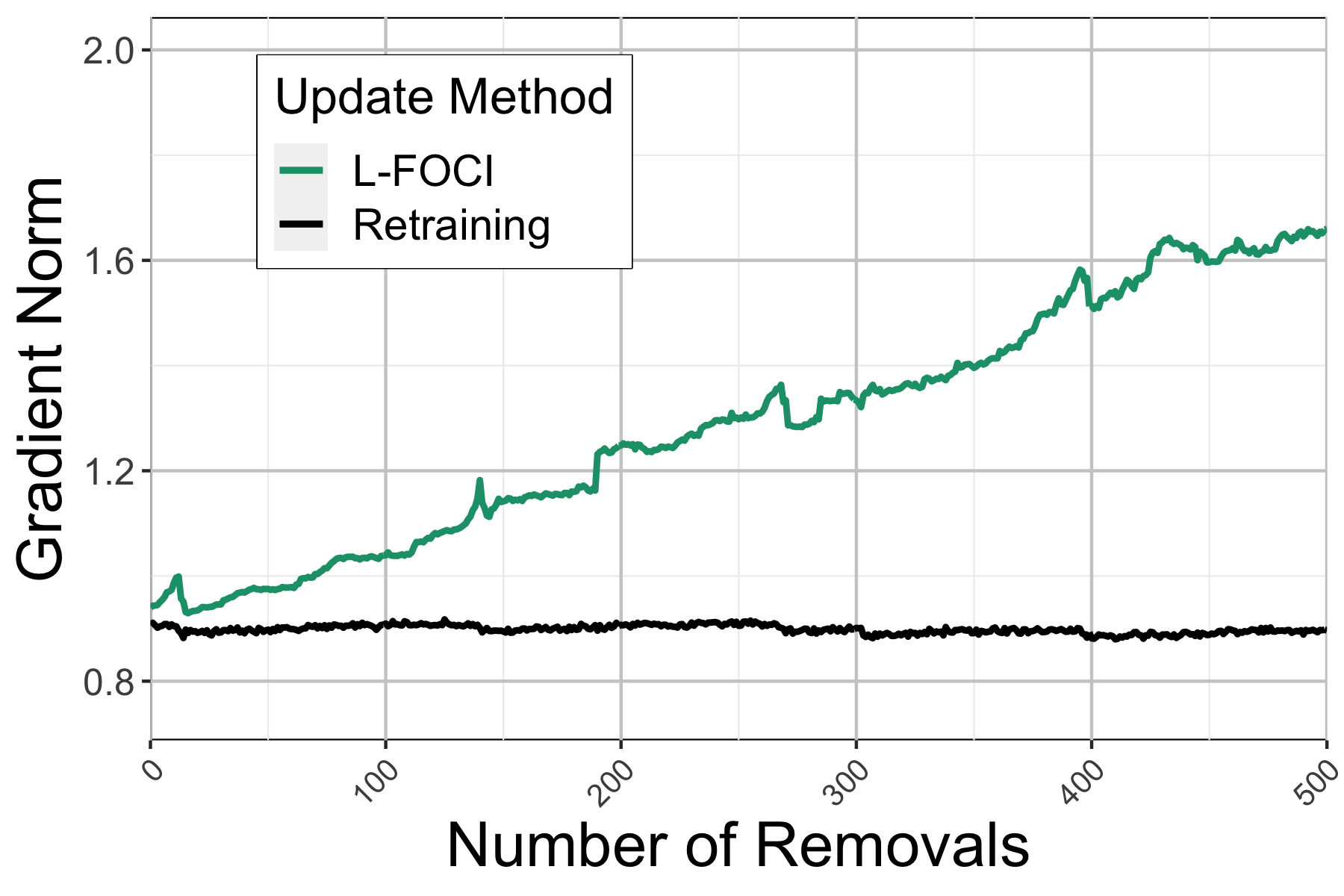}
    \includegraphics[width=0.24\columnwidth]{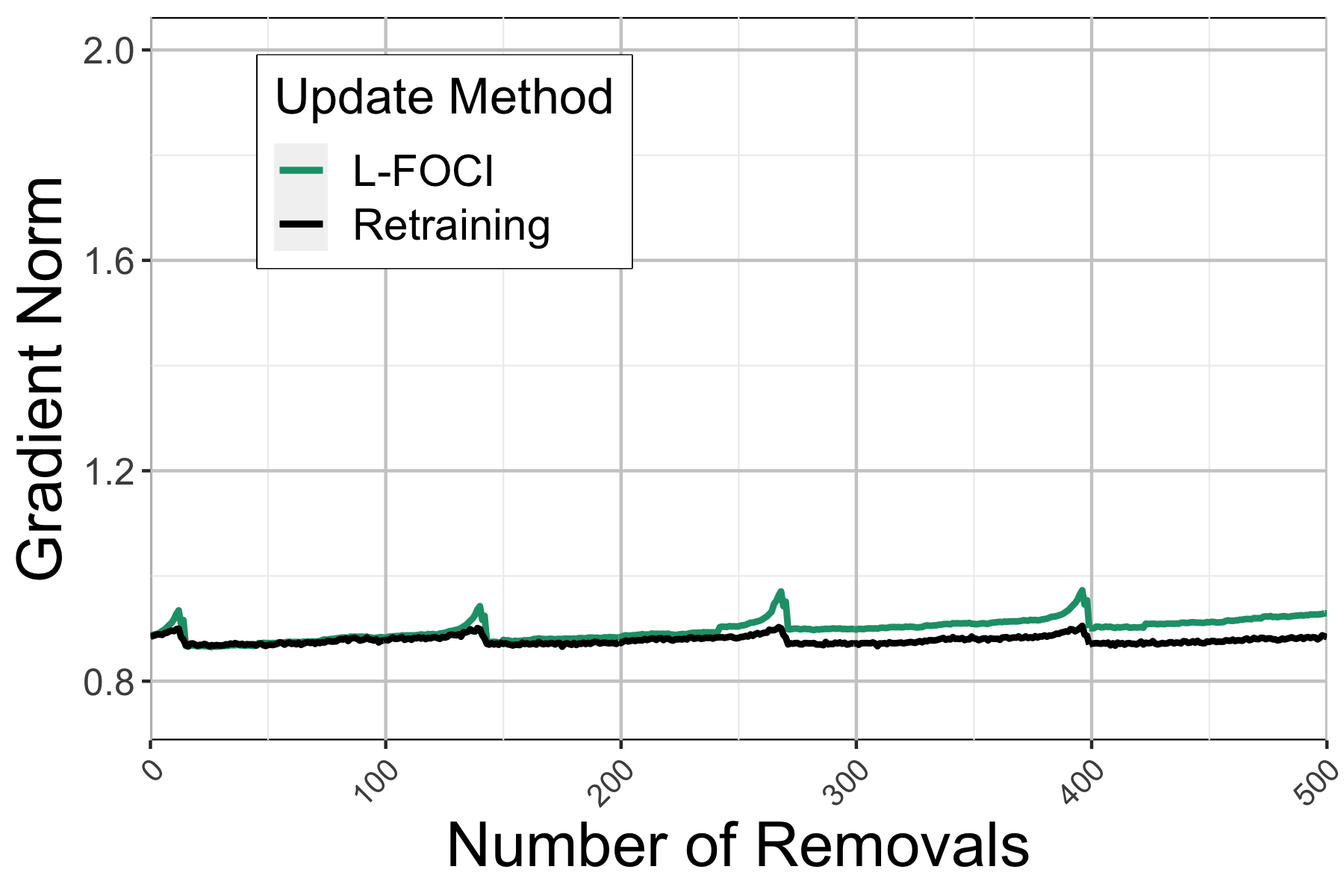}
    \caption{MNIST Retraining results, comparing the effect of weight decay on unlearning via our LFOCI unlearning scheme and retraining.}
    \label{fig:mnistretrainweightdecay}
\end{figure}

\subsubsection{CIFAR Retraining Comparisons: A Note on Batch Normalization}
An important requirement for our retraining experiment is that our residual training set used for both scrubbing validation and retraining is able to take on any size, including 1 and any size for which the modulus over the batch size equals 1. This causes particular problems when models include batch normalization layers: general practice in training deep neural networks includes the choice of dropping the last batch, so as to avoid issues of unbalanced batch sizes. For our setting we \textit{cannot} drop these batches, because we explicitly want to measure and compute on networks trained with and without specific samples. While we can ``skip" removals during our experimentation, this can still lead to odd behavior, see Figure~\ref{fig:cifarretrain}. The spikes are exactly congruent with points in the removal process corresponding to a final batch size of 1 for retraining. In general, care must be taken when attempting to unlearn from batchnorm models, and further work may be necessary to adequately address it, both in theory and practice.
\begin{figure}
    \centering
    \includegraphics[width=0.4\columnwidth]{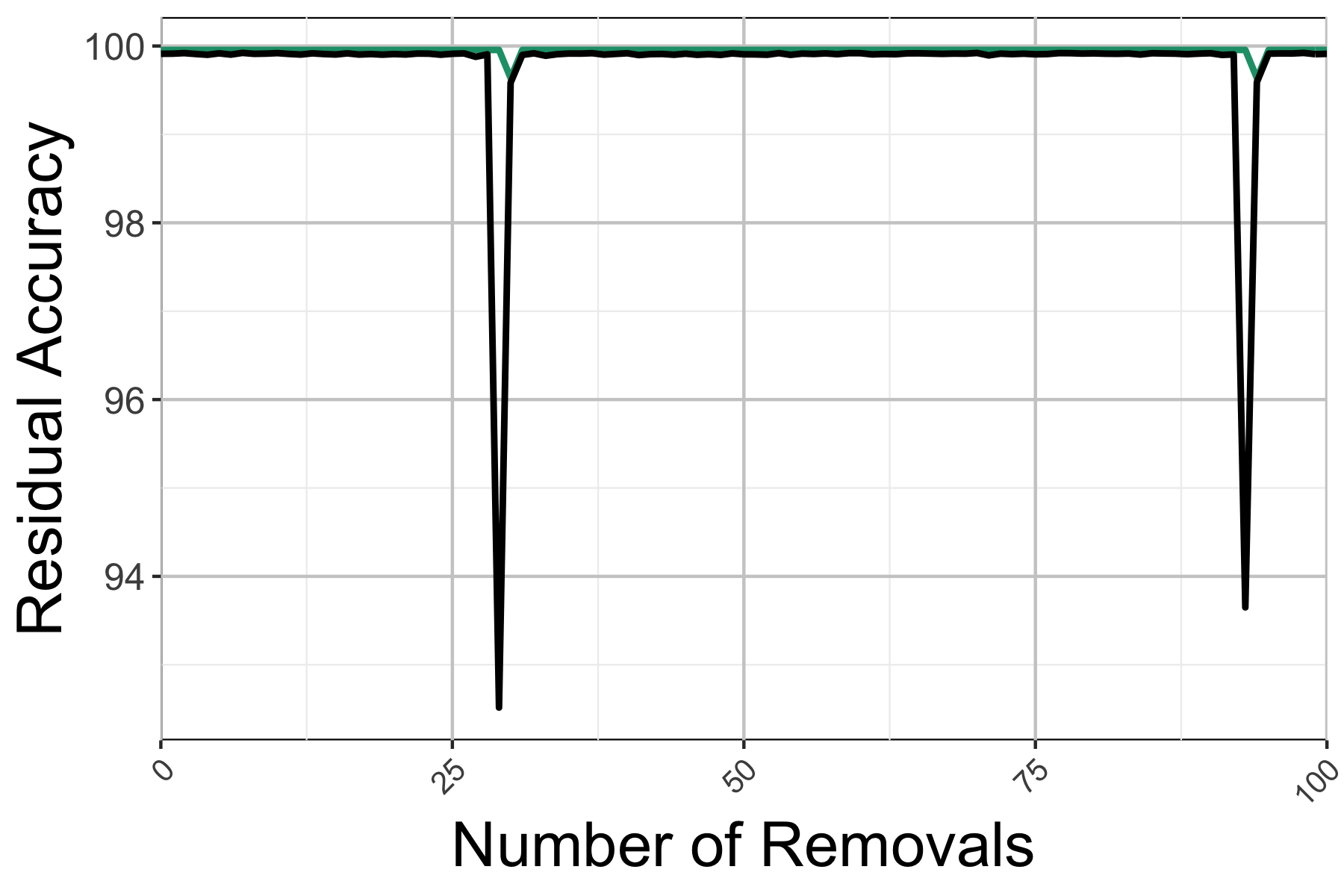}
    \includegraphics[width=0.4\columnwidth]{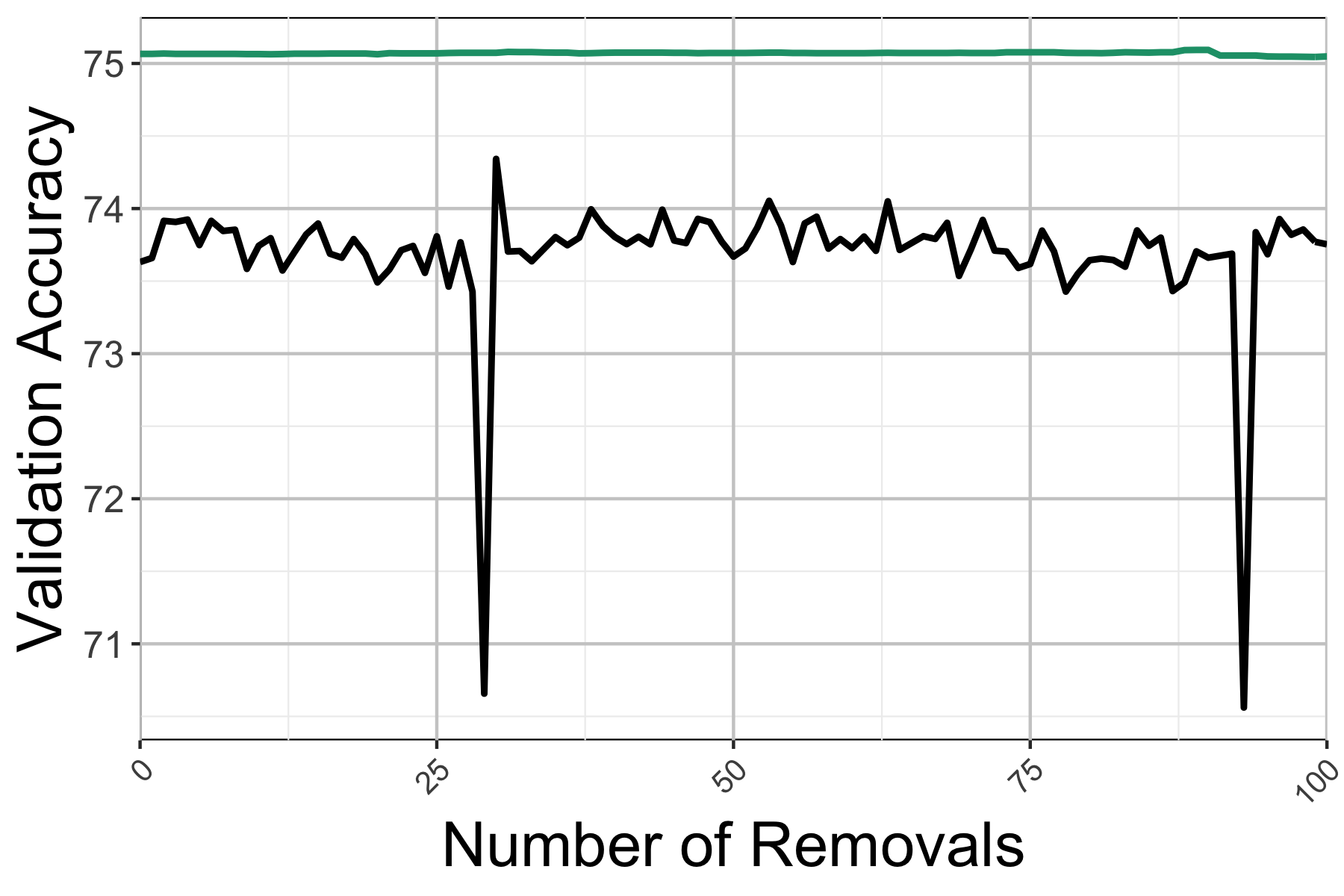}
    \caption{Retraining Results on CIFAR. Dips occur at removal counts where the modulus equals 1.}
    \label{fig:cifarretrain}
\end{figure}

\subsection{CIFAR-10 Model Comparisons}
Models were trained using Torch Hub, with a batch size of 64, learning rate of 0.1 for all models except VGG-11/bn, for which 0.01 was used. Data augmentation was NOT used, and weight decay was set to $0.01$.
1000 perturbations were used for distribution approximation. Privacy parameters were set to $\epsilon=0.1, \delta=0.01$.
Figures and numbers in the main paper were averaged over 2 replications, for a random choice of 1000 samples to scrub.

\subsection{LEDGAR DistillBERT Details}
For the NLP experiments, we used a pretrained model from HuggingFace as a starting point. Specifically, we used the transformer model ``distilbert-base-uncased", \url{https://huggingface.co/distilbert-base-uncased} which is a distilled version of the BERT base model, smaller and faster than BERT. It was pretrained on the same corpus in a self-supervised fashion, using the BERT base model as a teacher. DistilBERT \cite{sanh2019distilbert} was pretrained on the raw texts only, without any human labels. The three losses used for the pre-training are that of distillation loss, masked language modelling and cosine embedding loss. This pre-trained model was then fine-tuned for the downstream task of provision classification using the LEDGAR dataset introduced in \cite{tuggener2020ledgar}. We used the prototypical dataset which had 13 most common labels based on frequency. The model was fine-tuned for 4 epochs, updating all of it's parameters without any freezing based on binary cross entropy loss with class weighting. The labels were converted to one-hot vectors and hence binary cross entropy loss was used. Learning rate used was $5e^{-5}$ and weight decay of $0.01$. No weight decay was applied for bias and normalization layer parameters. We used batch size of 256 and restricted the maximum length of tokens to $128$ per data point. Further gradients were clipped based on the infinity norm to a value of $1.0$. We used AdamW optimizer with an epsilon value of $1e^{-8}$; and the learning rate scheduler used was WarmupLinearSchedule both from PyTorch\_Transformers. 

For unlearning experiments on this model, we remove provisions pertaining to a specific class. We removed samples from two classes namely ``Governing Laws" and ``Terminations" which had the highest and lowest support respectively. We were able to removed a varying number of samples from these classes based on the selection of the privacy parameter of $\epsilon$ for scrubbing. The results are tabulated in the main paper.

\subsection{VGG-Face Identification Scrubbing}
For this setting, the trained model was downloaded from \url{https://www.robots.ox.ac.uk/~vgg/data/vgg_face/} and converted to PyTorch via \url{https://github.com/prlz77/vgg-face.pytorch}. A partial version of the dataset was constructed using the list of image URLs, consisting of 100 images for each identity within the set. The images were processed as described in the original paper \cite{Parkhi15}.

Fine tuning was done for 4 epochs to estimate the Hessian for the sample downloaded using SGD with a learning rate of $0.0001$ and a weight decay/$l_2$ regularization of $0.01$, with a batch size of 16.

For unlearning, 100 images for a specific identity were randomly ordered and removed with $\epsilon=0.0001, \delta=0.01$. 100 perturbations were used to estimate the activation and loss distributions for L-FOCI.

\subsection{Person Re-identification}
We discuss in more detail the experimental details of unlearning deep neural networks for the person re-identification task in this section. We consider four different datasets namely, Market1501 \cite{zheng2015scalable}, MSMT17 \cite{wei2018person}, PRID \cite{hirzer11} and  QMUL GRID \cite{loy2009multi}. We unlearn from different deep neural networks including ResNet50, a variant of ResNet50 with a fully connected layer (called Resnet50\_fc512), Multi-Level Factorisation Net (MLFN) and MobilNet\_V2. In all cases the models were first trained to reasonable accuracy as per benchmarks before proceeding with unlearning a randomly selected individual's identity from the corresponding dataset. To perform experiments pertaining to person re-identification we make use of the popular framework torchreid \cite{zhou2019omni}. We had to make changes to the original code in order to make our procedure function correctly in this framework. We include this modified source within the code presented in the supplement. We use Adam as the optimizer, a step scheduler and learning rate of $0.0003$ across all person re-identification datasets and models used. We use softmax loss and weights were initialized using a model pre-trained on ImageNet in all cases. Images were resized to $256 \times 128$ before being used as input to any of the models. The number of training epochs was chosen accordingly to allow the training to have converged. Results from multiple runs involving different models, datasets and the privacy parameter $(\epsilon)$ are conclusive. With lower value of $\epsilon$, e.g. $0.0005$, the number of samples that could be unlearned for a particular class while maintaining model performance was lower than what could be unlearned for a higher value of the privacy parameter $\epsilon$, e.g. $0.1$. For the smaller datasets, i.e. PRID and QMUL GRID, which have approximately 2 samples per class, the unlearning procedure lead to more drastic changes as expected and it could be observed that our selection procedure selected many more layers to update than what it did for the larger datasets. Activation maps from some experiments are presented in Fig \ref{fig:reid_sup}.

\begin{figure*}
    \centering
    
    \begin{subfigure}[b]{0.98\textwidth}
        \centering
        \includegraphics[width=0.45\columnwidth]{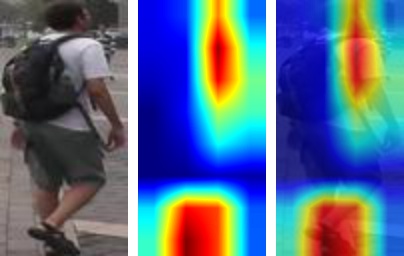}
        \hspace{15pt}
        \includegraphics[width=0.45\columnwidth]{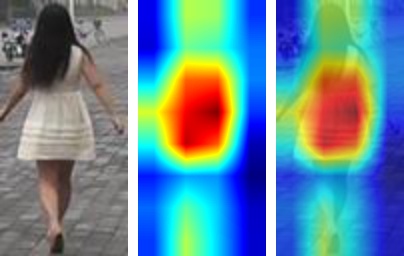}
    \end{subfigure}
    
    \vspace{10 pt}
    
    \begin{subfigure}[b]{0.98\textwidth}
        \centering
        \includegraphics[width=0.45\columnwidth]{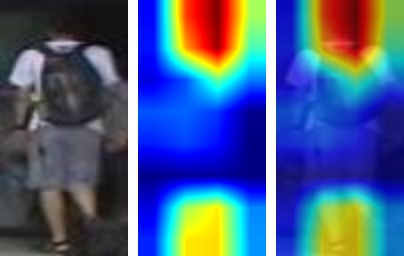}
        \hspace{15pt}
        \includegraphics[width=0.45\columnwidth]{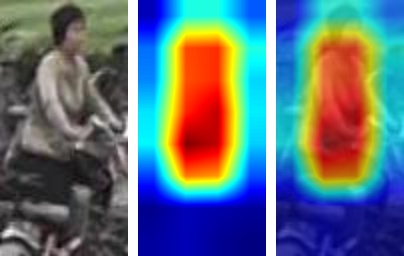}
    \end{subfigure}
    
    \vspace{10 pt}
    
    \begin{subfigure}[b]{0.98\textwidth}
        \centering
        \includegraphics[width=0.45\columnwidth]{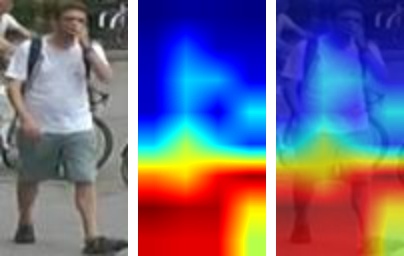}
        \hspace{15pt}
        \includegraphics[width=0.45\columnwidth]{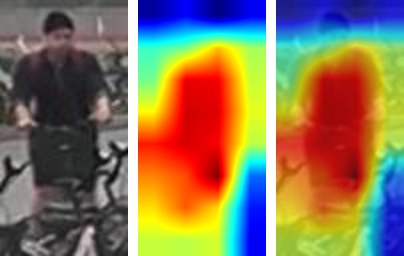}
    \end{subfigure}
    
    \vspace{10 pt}
    
    \begin{subfigure}[b]{0.98\textwidth}
        \centering
    \subfloat[\centering Scrubbed class]{{\includegraphics[width=0.45\columnwidth]{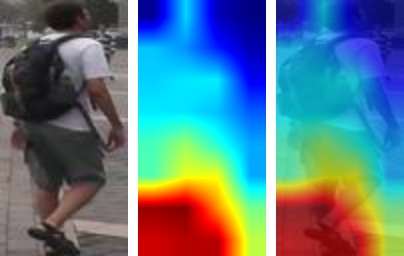}}}%
    \hspace{15pt}
    \subfloat[\centering Unscrubbed class]{{ \includegraphics[width=0.45\columnwidth]{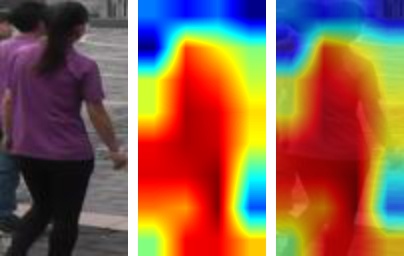}}}%
    
        
    \end{subfigure}
    
    
        
    
    
    \caption{\label{fig:reid_sup}Activation maps from a different models (top two rows correspond to MLFN and bottom two correspond to MobileNet\_V2; both trained on Market\_1501) scrubbed for the person on the left (right set is not scrubbed). For each triplet, from (L to R) are the original image, the activation map and its image overlay. Note the effect of scrubbing: activations change significantly for the scrubbed sample (compare column 2 to 3) whereas remain stable for the non-scrubbed sample (compare column 5 to 6).}
\end{figure*}

\section{Conditional Independence and Parameter Selection via L-CODEC}
Our algorithm is directly adapted from \cite{codec} to our parameter selection setting. Algorithm~\ref{alg:lfoci} shows the procedure, described for arbitrary random variables in Section 5 of \cite{codec}. 

While tests for independence exist, CODEC directly estimates explanation of variance with and without the conditional variable(s) of interest. For readers interested in conditional independence more generally, and statistical and theoretical foundations, please see \cite{spirtes2000causation, dawid1979conditional}. More recent information-based formulations can be found in \cite{bullseye} and references therein.

\begin{algorithm}
\SetAlgoLined
\KwData{$z'$, the full parameter set $w$ indexed by $\Theta := \{1,\ldots,d\}$}
\KwResult{Sufficient set $P \subseteq \Theta$}
Identify $p \in \Theta$ that maximizes $T(z',w_p)$ \\
Set $P = \{p\}$ \\
\While{$T(z',w_{\Theta\setminus P}|w_P) > 0$}{
    Identify $p \in \Theta\setminus P$ that maximizes $T(z',w_{\Theta\setminus P}| w_{\Theta \cup s})$ \\
    \eIf{$T(z',w_{\setminus P}, w_{P \cup p}) < 0$}{break}{Append $P = P \cup p$}
}
 \caption{\label{alg:lfoci} Parameter MB Identification via L-CODEC (L-FOCI)}
\end{algorithm}
\section{Alternate Hessian Approximations}

Typical approximations are non often non-sparse; a key focus of our proposal is a reasonably informed sparse estimation in deep unlearning: we cannot allocate both full networks and the space for an inverse for $50$K$+$ parameters (needs $10+$GB alone). For Deep unlearning specifically, our sub selection makes this possible. Diagonal modification still needs full parameter updates. However, we explored the utilization of other Hessian inverse approximation schemes. More specifically, we implemented an unlearning scheme based on Kronecker-Factored Approximate Curvature (K-FAC) \cite{martens2015optimizing} which exploits an efficient invertible approximation of a deep learning model's Fisher information matrix which can be non-sparse and neither low rank nor diagonal. In an experimental setup, we perform unlearning based on K-FAC from an multi layer perceptron model trained on MNIST dataset. We don't see any observable updates happening to the model based on validation metrics. Whereas, the exact same model with the exact same set of parameters can unlearn the same set of data-points using our proposed deep unlearning method based on LCODEC. We would like to point out that in order to unlearn from deep models using existing approximations schemes like K-FAC, we might have to re-imagine the update step. This demands further investigation. In other words our procedure may not be more broadly applicable to non-sparse general Hessian inverse approximations without an obvious CI structure.

We have included the code to compare K-FAC based unlearning with LCODEC based unlearning in the file \url{https://github.com/vsingh-group/LCODEC-deep-unlearning/blob/main/scrub/kfac\_scrub.py}

In our implementation we heavily rely on the KFAC approximations of the Hessian as provided in \url{https://github.com/cybertronai/autograd-lib} . More instructions can be found in the README.




\end{document}